\documentclass{article}
\usepackage[accepted]{icml2021_arxiv}
\usepackage[utf8]{inputenc}
\usepackage{url}
\usepackage{hyperref}
\usepackage{adjustbox}
\usepackage{multirow}

\usepackage{amsmath,amsfonts,bm,amssymb,amsthm,yfonts}

\def\eqref#1{equation~\ref{#1}}

\def\1{\bm{1}}

\def\rr{{\textnormal{r}}}

\def\vzero{{\bm{0}}}
\def\vone{{\bm{1}}}

\def\va{{\bm{a}}}
\def\vb{{\bm{b}}}

\def\ve{{\bm{e}}}
\def\vf{{\bm{f}}}
\def\vh{{\bm{h}}}
\def\vg{{\bm{g}}}
\def\vh{{\bm{h}}}

\def\vq{{\bm{q}}}

\def\vs{{\bm{s}}}

\def\vu{{\bm{u}}}
\def\vv{{\bm{v}}}
\def\vw{{\bm{w}}}
\def\vx{{\bm{x}}}
\def\vy{{\bm{y}}}

\def\rr{\mathbb{R}}

\DeclareMathAlphabet{\mathsfit}{\encodingdefault}{\sfdefault}{m}{sl}
\SetMathAlphabet{\mathsfit}{bold}{\encodingdefault}{\sfdefault}{bx}{n}

\def\t{\intercal}

\newcommand{\Cov}{\mathrm{Cov}}

\def\eee#1#2{\mathbb{E}_{#1}\left[#2\right]}

\newtheorem{theorem}{Theorem}
\newtheorem{definition}{Definition}
\newtheorem{lemma}{Lemma}
\newtheorem{corollary}{Corollary}
\newtheorem{assumption}{Assumption}

\usepackage{listings}
\usepackage{accents}

\def\tr{\mathrm{tr}}
\def\var{\mathbb{V}}
\def\dd{\mathrm{d}} 
\def\circle#1{{\small \textcircled{\raisebox{-0.9pt}{#1}}}}

\def\aug{\mathrm{aug}}
\def\Cov{\mathrm{Cov}}

\def\vec{\mathrm{vec}}
\def\ee#1{\mathbb{E}\left[#1\right]}

\def\ema{\mathrm{ema}}

\def\eee#1#2{\mathbb{E}_{#1}\left[#2\right]}
\def\cW{\mathcal{W}}

\def\hltm{\texttt{HLTM}}

\usepackage{tabularx,colortbl,xcolor}

\newif\ifcomments
\commentsfalse

\newif\ifarxiv
\arxivfalse

\def\vspacenoarxiv#1{\ifarxiv \else \vspace{#1} \fi}

\ifcomments
\usepackage[textsize=small,color=lightgray]{todonotes}
\paperwidth=\dimexpr \paperwidth + 10cm\relax
\paperheight=\dimexpr \paperheight + 5cm\relax
\oddsidemargin=\dimexpr\oddsidemargin + 5cm\relax
\evensidemargin=\dimexpr\evensidemargin + 5cm\relax
\marginparwidth=\dimexpr \marginparwidth + 5cm\relax
\else
\usepackage[disable]{todonotes}
\fi

\newcommand{\yuandong}[1]{\todo[fancyline, color=red!50]{\textbf{Yuandong}: #1}\ignorespaces}
\newcommand{\yuandongil}[1]{\todo[inline, color=red!50]{\textbf{Yuandong}: #1}\ignorespaces}
\newcommand{\lantao}[1]{\todo[fancyline, color=cyan!50]{\textbf{Lantao}: #1}\ignorespaces}
\newcommand{\lantaoil}[1]{\todo[inline, color=cyan!50]{\textbf{Lantao}: #1}\ignorespaces}
\newcommand{\xinlei}[1]{\todo[fancyline, color=purple!50]{\textbf{Xinlei}: #1}\ignorespaces}

\def\ourtitle{Understanding Self-supervised Learning with Dual Deep Networks}

\icmltitlerunning{\ourtitle}

\begin{document}

\twocolumn[
\icmltitle{\ourtitle}

\icmlsetsymbol{equal}{*}

\begin{icmlauthorlist}
\icmlauthor{Yuandong Tian}{fair}
\icmlauthor{Lantao Yu}{stanford}
\icmlauthor{Xinlei Chen}{fair}
\icmlauthor{Surya Ganguli}{fair,stanford}
\end{icmlauthorlist}

\icmlaffiliation{fair}{Facebook AI Research}
\icmlaffiliation{stanford}{Stanford University}

\icmlcorrespondingauthor{Yuandong Tian}{yuandong@fb.com}

\icmlkeywords{Machine Learning, ICML}

\vskip 0.3in
]

\printAffiliationsAndNotice{}

\vspacenoarxiv{-0.1in}
\begin{abstract}
\vspacenoarxiv{-0.03in}
We propose a novel theoretical framework to understand contrastive self-supervised learning (SSL) methods that employ dual pairs of deep ReLU networks (e.g., SimCLR). First, we prove that in each SGD update of SimCLR with various loss functions, including simple contrastive loss, soft Triplet loss and InfoNCE loss, the weights at each layer are updated by a \emph{covariance operator} that specifically amplifies initial random selectivities that vary across data samples but survive averages over data augmentations. To further study what role the covariance operator plays and which features are learned in such a process, we model data generation and augmentation processes through a \emph{hierarchical latent tree model} (\hltm) and prove that the hidden neurons of deep ReLU networks can learn the latent variables in the \hltm, despite the fact that the network receives \emph{no direct supervision} from these unobserved latent variables. This leads to a provable emergence of hierarchical features through the amplification of initially random selectivities through contrastive SSL. Extensive numerical studies justify our theoretical findings. Code is released in \url{https://github.com/facebookresearch/luckmatters/tree/master/ssl}.
\end{abstract}

\vspacenoarxiv{-0.15in}
\section{Introduction}
While self-supervised learning (SSL) has achieved great empirical success across multiple domains, including computer vision~\cite{he2020momentum,goyal2019scaling,simclr,byol,misra2020self,caron2020unsupervised}, natural language processing~\cite{bert}, and speech recognition~\cite{wu2020self,baevski2020effectiveness,baevski2019vq}, its theoretical understanding remains elusive, especially when multi-layer nonlinear deep networks are involved~\cite{Bahri2020-mi}. Unlike supervised learning (SL) that deals with labeled data, SSL learns meaningful structures from randomly initialized networks without human-provided labels. 

In this paper, we propose a systematic theoretical analysis of SSL with deep ReLU networks. Our analysis imposes no parametric assumptions on the input data distribution and is applicable to state-of-the-art SSL methods that typically involve two parallel (or \emph{dual}) deep ReLU networks during training (e.g., SimCLR~\cite{simclr}, BYOL~\cite{byol}, etc). We do so by developing an analogy between SSL and a theoretical framework for analyzing supervised learning, namely the student-teacher setting ~\cite{tian2019student,allen2020backward,Lampinen2018-sl,saad1996dynamics}, which also employs a pair of dual networks.  Our results indicate that SimCLR weight updates at every layer are amplified by a fundamental positive semi definite (PSD) \emph{covariance operator} that only captures feature variability across data points that {\it survive} averages over data augmentation procedures designed in practice to scramble semantically unimportant features (e.g. random image crops, blurring or color distortions ~\cite{falcon2020framework,kolesnikov2019revisiting,misra2020self,purushwalkam2020demystifying}). This covariance operator provides a principled framework to study how SimCLR amplifies initial random selectivity to obtain distinctive features that vary \emph{across} samples after surviving averages over data-augmentations. 

While the covariance operator is a mathematical object that is valid for any data distribution and augmentations, we further study its properties under specific data distributions and augmentations. We first start with a simple one-layer case where two 1D objects undergo 1D translation, then study a fairly general case when the data are generated by a hierarchical latent tree model (\hltm), which can be regarded as an abstract conceptual model for object compositionality~\cite{} in computer vision. In this case, training deep ReLU networks on the data generated by the \hltm{} leads to the emergence of learned representations of the latent variables in its intermediate layers, even if these intermediate nodes have \emph{never} been directly supervised by the unobserved and inaccessible latent variables. This shows that in theory, useful hidden features can automatically emerge by contrastive self-supervised learning.  

To the best of our knowledge, we are the first to provide a systematic theoretical analysis of modern SSL methods with deep ReLU networks that elucidates how both data and data augmentation, drive the learning of internal representations across multiple layers.  

\section{Related Works}
In addition to SimCLR and BYOL, many concurrent SSL frameworks exist to learn good representations for computer vision tasks. MoCo~\cite{he2020momentum,chen2020improved} keeps a large bank of past representations in a queue as the slow-progressing target to train from.  DeepCluster~\cite{caron2018deep} and SwAV~\cite{caron2020unsupervised} learn the representations by iteratively or implicitly clustering on the current representations and improving representations using the cluster label. ~\cite{alwassel2019self} applies similar ideas to multi-modality tasks. Contrastive Predictive Coding~\cite{oord2018representation} learns the representation by predicting the future of a sequence in the latent
space with autoregressive models and InfoNCE loss. Contrastive MultiView Coding~\cite{tian2019contrastive} uses multiple sensory channels (called  ``views'') of the same scene as the positive pairs and independently sampled views as the negative pairs to train the model. Recently, \cite{li2020prototypical} moves beyond instance-wise pairs and proposes to use prototypes to construct training pairs that are more semantically meaningful.  

In contrast, the literature on the (theoretical) analysis of SSL is sparse.~\cite{wang2020understanding} shows directly optimizing the alignment/uniformity of the positive/negative pairs leads to comparable performance against contrastive loss.  
\cite{Arora2019-xw} proposes an interesting analysis of how contrastive learning aids downstream classification tasks, given assumptions about data generation.
~\cite{lee2020predicting} analyzes how learning pretext tasks could reduce the sample complexity of the downstream task and~\cite{tosh2020contrastive} analyzes contrastive loss with multi-view data in the semi-supervised setting, with different generative models.
However, they either work on linear models or treat deep models as a black-box function approximators with sufficient capacity. In comparison, we incorporate self-supervision, deep models, contrastive loss, data augmentation and generative models together into the same theoretical framework, and make an attempt to understand how and what intermediate features emerge in modern SSL architectures with deep models that achieve SoTA. 

\def\rev{reversible}
\def\revnoun{reversibility}

\vspacenoarxiv{-0.1in}
\section{Overall framework}
\label{sec:background}

\paragraph{Notation.} Consider an $L$-layer ReLU network obeying $\vf_l = \psi(\tilde\vf_l)$ and $\tilde\vf_l = W_l \vf_{l-1}$ for $l=1,\dots L$. Here $\tilde\vf_l$ and $\vf_l$ are $n_l$ dimensional pre-activation and activation vectors in layer $l$, with $\vf_0 = \vx$ being the input and $\vf_L=\tilde\vf_L$ the output (no ReLU at the top layer). $W_l \in \rr^{n_l \times n_{l-1}}$ are the weight matrices, and $\psi(\vu) := \max(\vu, 0)$ is the element-wise ReLU nonlinearity. We let $\cW := \{W_l\}_{l=1}^L$ be all network weights.  We also denote the gradient of any loss function with respect to $\vf_l$ by $\vg_l \in \rr^{n_l}$, and the derivative of the output $\vf_L$ with respect to an earlier pre-activation $\tilde\vf_l$ by the Jacobian matrix $J_l(\vx;\cW) \in \rr^{n_L\times n_l}$, as both play key roles in backpropagation (Fig.~\ref{fig:notation}(b)).

\begin{figure*}
    \centering
    \includegraphics[width=0.9\textwidth]{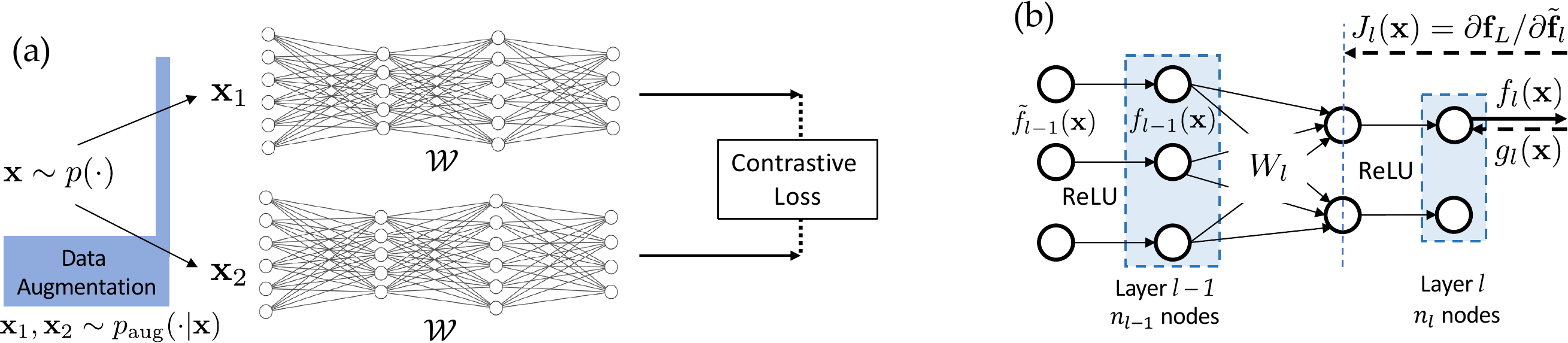}
    \vspace{-0.1in}
    \caption{\small \textbf{(a)} Overview of the SimCLR architecture. A data point $\vx\sim p(\cdot)$ is augmented to two views $\vx_1,\vx_2\sim p_\aug(\cdot|\vx)$, which are sent to two deep ReLU networks with identical weights $\cW$, and their outputs are sent to contrastive loss function. \textbf{(b)} Detailed notations.}
    \vspace{-0.1in}
    \label{fig:notation}
\end{figure*}

\vspacenoarxiv{-0.05in}
\paragraph{An analogy between self-supervised and supervised learning: the dual network scenario.} Many recent successful approaches to self-supervised learning (SSL), including SimCLR~\cite{simclr}, BYOL~\cite{byol} and MoCo \cite{he2020momentum}, employ dual ``Siamese'' pairs ~\cite{koch2015siamese} of such networks (Fig.~\ref{fig:notation}(b)).  Each network has its own set of weights $\cW_1$ and $\cW_2$, receives respective inputs $\vx_1$ and $\vx_2$ and generates outputs $\vf_{1,L}(\vx_1;\cW_1)$ and $\vf_{2,L}(\vx_2;\cW_2)$.  The pair of inputs $\{\vx_1,\vx_2\}$ can be either positive or negative, depending on how they are sampled.  For a positive pair, a {\it single} data point $\vx$ is drawn from the data distribution $p(\cdot)$, and then two augmented views $\vx_1$ and $\vx_2$ are drawn from a conditional augmentation distribution  $p_\aug(\cdot|\vx)$.  Possible image augmentations include random crops, blurs or color distortions, that ideally preserve semantic content useful for downstream tasks.  In contrast, for a negative pair, two {\it different} data points $\vx, \vx' \sim p(\cdot)$ are sampled, and then each are augmented independently to generate $\vx_1 \sim p_\aug(\cdot|\vx)$ and $\vx_2 \sim p_\aug(\cdot|\vx')$.  For SimCLR, the dual networks have tied weights with $\cW_1 = \cW_2$, and a loss function is chosen to encourage the representation of positive (negative) pairs to become similar (dissimilar). 

Our fundamental goal is to analyze the mechanisms governing how contrastive SSL methods like SimCLR lead to the emergence of meaningful intermediate features, starting from random initializations, and how these features depend on the data distribution $p(\vx)$ and augmentation procedure $p_\aug(\cdot|\vx)$. Interestingly, the analysis of {\it supervised} learning (SL) often employs a similar dual network scenario, called the \emph{teacher-student setting}~\cite{tian2019student,allen2020backward,Lampinen2018-sl,saad1996dynamics}, where $\cW_2$ are the ground truth weights of a {\it fixed} teacher network, which generates outputs in response to random inputs.  These input-output pairs constitute training data for the first network, which is a student network. Only the student network's weights $\cW_1$ are trained to match the target outputs provided by the teacher.  This yields an interesting mathematical parallel between SL, in which the teacher is fixed and only the student evolves, and SSL, in which both the teacher and student evolve with potentially different dynamics.  This mathematical parallel opens the door to using techniques from SL (e.g.,~\cite{tian2019student}) to analyze SSL.   

\vspacenoarxiv{-0.05in}
\paragraph{Gradient of $\ell_2$ loss for dual deep ReLU networks.} As seen above, the (dis)similarity of representations between a pair of dual networks plays a key role in both SSL and SL.  We thus consider minimizing a simple measure of dissimilarity, the squared $\ell_2$ distance $r := \frac{1}{2}\|\vf_{1,L}-\vf_{2,L}\|^2$ between the final outputs $\vf_{1,L}$ and $\vf_{2,L}$ of two multi-layer ReLU networks with weights $\cW_1$ and $\cW_2$ and inputs $\vx_1$ and $\vx_2$. This gradient formula will be used to analyze multiple contrastive loss functions in Sec.~\ref{sec:simclr}. Without loss of generality, we only analyze the gradient w.r.t $\cW_1$. For each layer $l$, we first define the  \emph{connection} $K_l(\vx)$, a quantity that connects the bottom-up feature vector $\vf_{l-1}$ with the top-down Jacobian $J_l$, which both contribute to the gradient at weight layer $l$.


\begin{definition}[The connection $K_l(\vx)$]
\label{def:connection}
The \emph{connection} $K_l(\vx;\cW):= \vf_{l-1}(\vx;\cW)\otimes J_l^\t(\vx;\cW) \in \rr^{n_l n_{l-1}\times n_L}$. Here $\otimes$ is the Kronecker product.
\end{definition}

\begin{theorem}[Squared $\ell_2$ Gradient for dual deep ReLU networks]
\label{thm:l2-two-tower}
The gradient $g_{W_l}$ of $r$ w.r.t. $W_l \in \rr^{n_l \times n_{l-1}}$ for a single input pair $\{\vx_1, \vx_2\}$ is  
(here $K_{1,l} := K_l(\vx_1;\cW_1)$, $K_{2,l} := K_l(\vx_2;\cW_2)$ and $g_{W_l}:=\vec(\partial r/\partial W_{1,l})$):
\begin{equation}
    \vspace{-0.1in}
    g_{W_l}\!=\! K_{1,l}\left[K^\t_{1,l}\vec(W_{1,l})\!-\!K^\t_{2,l}\vec(W_{2,l})\right]
\end{equation}
\end{theorem}
Here $\vec(W)$ is a column vector constructed by stacking columns of $W$ together. We used such notation for the gradient $g_{W_l}$ and weights $W_l$ to emphasize certain theoretical properties of SSL learning below. The equivalent matrix form is  
$\partial r/\partial W_{1,l} = 
J_{1,l}^\t \left[ 
J_{1,l} W_{1,l} \vf_{1,l-1}  -
J_{2,l} W_{2,l} \vf_{2,l-1} 
\right] \vf_{1,l-1}^T $. 

See Appendix for proofs of all theorems in the main text. 

\ifarxiv
\textbf{Remarks}. Note that Theorem~\ref{thm:l2-two-tower} can also be applied to networks with other non-linearity than ReLU, as long as they satisfy a more general property called \emph{\revnoun{}} (See Def.~\ref{def:reversibility} in Appendix). This includes linear, LeakyReLU, and monomial activations, such as the quadratic activation function $\psi(x) = x^2$ used in many previous works~\cite{du2018power,soltanolkotabi2018theoretical,allen2020backward}. In this paper, we focus on ReLU and leave other non-linearity for future works. 
\fi

\def\simple{\mathrm{simp}}
\def\triplet{\mathrm{tri}}
\def\nce{\mathrm{nce}}

\def\lsimple{L_{\simple}}
\def\ltriplet{L_{\triplet}}
\def\lnce{L_{\nce}}

\vspacenoarxiv{-0.1in}
\section{The Covariance Operator}
\label{sec:simclr}
\xinlei{One thing to note from the empirical side is that one can (actually according to their code, they did) combine the forward pass of both views into one single forward pass as the network weight is identical for both views.}
As discussed above, SimCLR~\cite{simclr} employs both positive and negative input pairs, and a symmetric network structure with $\cW_1 = \cW_2 = \cW$. Let $\{\vx_1, \vx_+\}$ be a positive input pair from $\vx$, and let
$\{\vx_1, \vx_{k-}\}$ for $k=1,\dots,H$ be $H$ negative pairs.  
These input pairs induce corresponding squared $\ell_2$ distances in output space,
$r_+ := \frac{1}{2}\|\vf_{1,L} - \vf_{+,L}\|^2_2$, and 
$r_{k-} := \frac{1}{2}\|\vf_{1,L} - \vf_{k-,L}\|^2_2$. 

In this paper, we consider three different contrastive losses:
\begin{itemize}
    \item[\textbf{(1)}] the simple contrastive loss $\lsimple := r_+ - r_-$, 
    \item[\textbf{(2)}] (Soft) Triplet loss $\ltriplet^\tau:= \tau \log(1 + e^{(r_+ - r_- + r_0) / \tau})$ (here $r_0 \ge 0$ is the margin). Note that $\lim_{\tau\rightarrow 0}\ltriplet^\tau = \max(r_+ - r_- + r_0, 0)$~\cite{schroff2015facenet}
    \item[\textbf{(3)}] InfoNCE loss $\lnce^\tau$~\cite{oord2018representation}:
\begin{equation}
\vspacenoarxiv{-0.1in}
\!\!\!\lnce^\tau(r_+, \{r_{k-}\}_{k=1}^H) \!\!:=\!\!-\log\frac{e^{-r_+/\tau}}{e^{-r_+/\tau} + \sum_{k=1}^H e^{-r_{k-}/\tau}} \label{eq:contrastive-loss}
\end{equation}
\end{itemize}
Note that when $\|\vu\|_2 = \|\vv\|_2 = 1$, we have $-\frac{1}{2}\|\vu-\vv\|^2_2 = \mathrm{sim}(\vu,\vv) - 1$ where $\mathrm{sim}(\vu,\vv) = \frac{\vu^\t\vv}{\|\vu\|_2\|\vv\|_2}$, and Eqn.~\ref{eq:contrastive-loss} reduces to what the original SimCLR uses (the term $e^{-1/\tau}$ cancels out). 

For simplicity, we move the analysis of the final layer $\ell_2$ normalization to Appendix. When there is no $\ell_2$ normalization, the goal of our analysis is to show that useful weight components grow exponentially in the gradient updates.

We first note one interesting common property: 
\begin{theorem}[Common Property of Contrastive Losses]
\label{thm:contrastive-loss}
For loss functions $L \in \{\lsimple{}, \ltriplet^\tau{}, \lnce^\tau\}$, we have $\frac{\partial L}{\partial r_+} > 0$, $\frac{\partial L}{\partial r_{k-}} < 0$ for $1 \le k \le H$ and $\frac{\partial L}{\partial r_+} + \sum_{k=1}^H \frac{\partial L}{\partial r_{k-}} = 0$. 
\end{theorem}

\def\aug{\mathrm{aug}}
\def\covop{\mathrm{OP}}

With Theorem~\ref{thm:l2-two-tower} and Theorem~\ref{thm:contrastive-loss}, we now present our first main contribution of this paper: the gradient in SimCLR is governed by a positive semi-definite (PSD) \emph{covariance operator} at any layer $l$: 

\begin{theorem}[Covariance Operator for $\lsimple$]
\label{thm:contrast-simclr-simple}
With large batch limit, $W_l$'s update under $\lsimple{}$ is $W_l(t+1) = W_l(t) + \alpha \Delta W_l(t)$ ($\alpha$ is the learning rate), where 
\begin{equation}
\vec(\Delta W_l(t)) = \covop_l^\simple(\cW) \vec(W_l(t)). \label{eq:covariance-operator}
\end{equation}
Here $\covop_l^\simple(\cW)\!\!:=\!\!\var_{\vx\sim p(\cdot)}[\bar K_l(\vx;\cW)] \in \rr^{n_ln_{l-1}\times n_ln_{l-1}}$ is the \emph{covariance operator} for $\lsimple{}$, $\bar K_l(\vx;\cW) := \eee{\vx' \sim p_\aug(\cdot|\vx)}{K_l(\vx';\cW)}$ is the expected connection under the augmentation distribution, conditioned on datapoint $\vx$.
\end{theorem}

Intuitively, the \emph{covariance operator} $\covop_l(\cW)$ is a time-varying PSD matrix over the \emph{entire} training procedure. Therefore, all its eigenvalues are non-negative and at any time $t$, $W_l$ is most amplified along its largest eigenmodes.  Intuitively, $\covop_l(\cW)$ ignores different views of the same sample $\vx$ by averaging over the augmentation distribution to compute $\bar K_l(\vx)$, and then computes the expected covariance of this \emph{augmentation averaged} connection with respect to the data distribution $p(\vx)$.   Thus, at all layers, any variability in the connection across different data points, that survives augmentation averages, leads to weight amplification. This amplification of weights by the PSD data covariance of an augmentation averaged connection constitutes a fundamental description of SimCLR learning dynamics for \emph{arbitrary} data and augmentation distributions and holds at every layer of arbitrarily deep ReLU networks.

Given this result for $\lsimple{}$, one might ask whether the same property holds for more realistic loss functions like $\ltriplet{}$ and $\lnce{}$ that are extensively used in prior works. The answer is yes. Define \emph{weighted} covariance for matrices $X$ and $Y$: 
\begin{equation}
\Cov^\xi[X,Y]\!\!:=\!\!\ee{\xi(X,Y)(X-\ee{X})(Y-\ee{Y})^\t}
\end{equation}
and $\var^\xi[X] := \Cov^\xi[X,X]$. Note that $\Cov[X,Y]$ means $\xi(\cdot) \equiv 1$. Then we have the following theorem: 

\begin{theorem}[Covariance Operator for $\ltriplet^\tau{}$ and $\lnce^\tau$ ($H = 1$, single negative pair)]
\label{thm:contrast-simclr-pairwise}
Let $r := \frac{1}{2}\|\vf_L(\vx) - \vf_L(\vx')\|_2^2$. The covariance operator $\covop_l(\cW)$ has the following form:
\begin{equation}
\covop_l(\cW) = \frac{1}{2}\var^\xi_{\vx,\vx'\sim p(\cdot)}\left[\bar K_l(\vx) - \bar K_l(\vx')\right] + \theta \label{eq:weighted-cov-op}
\end{equation}
where the pairwise weight $\xi(r)$ takes the following form for different loss functions:
\begin{equation}
   \xi(r) = \left\{
   \begin{array}{cc}
   1 & L = \lsimple \\
    \frac{e^{-(r-r_0)/\tau}}{1 + e^{-(r-r_0)/\tau}} & L = \ltriplet^\tau \\
    \frac{1}{\tau} \frac{e^{-r/\tau}}{1 + e^{-r/\tau}} & L = \lnce^\tau{}
   \end{array}
   \right. \label{eq:xi}
\end{equation}
and $\theta\!:=\!\mathcal{O}(\eee{\vx,\vx'}{\sqrt{r(\vx,\vx')}\sigma_\aug(\vx)}\!+\!\eee{\vx}{\sigma^2_\aug(\vx)})$ is the residue term. $\sigma^2_\aug(\vx) := \tr\var_{\vx''\sim p_\aug(\cdot|\vx)}[\vf_L(\vx'')]$ and $\mathrm{tr}$ is the trace of a matrix. For $\lsimple$, $\theta\equiv 0$.  
\end{theorem}

Note that Theorem~\ref{thm:contrast-simclr-pairwise} is a strict extension of Theorem~\ref{thm:contrast-simclr-simple}: with constant pairwise weight $\xi(r(\vx,\vx'))\equiv 1$, $\frac{1}{2}\var_{\vx,\vx'}\left[\bar K_l(\vx) - \bar K_l(\vx')\right] = \var_{\vx}\left[\bar K_l(\vx)\right]$. 

Intuitively, the difference between $\lsimple{}$ and $\ltriplet{}$, $\lnce{}$ is that the last two put an emphasis on distinctive samples $\vx$ and $\vx'$ whose representations $\vf_{1,L}$ and $\vf_{2,L}$ are close (i.e., small $r(\vx,\vx')$). This makes sense since the goal of the contrastive learning is to learn a representation that separates distinct samples. 

\def\exact{\mathrm{exact}}

\vspace{-0.1in}
\subsection{Discussion and Remarks}
\vspace{-0.1in}
\textbf{The residue term $\theta$}. The residue term in Theorem~\ref{thm:contrast-simclr-pairwise} is related to the variance of the output representation within data augmentation, $\tr\var_{\vx''\sim p_\aug(\cdot|\vx)}[\vf_L(\vx'')]$. For $\lsimple{}$, the expression for covariance operator is exact and $\theta = 0$. While for $\ltriplet^\tau{}$ and $\lnce^\tau{}$, this term is nonzero, we expect it to shrink during the training so that these two losses drive learning more and more through a PSD operator at each layer $l$. Alternatively, we could construct a specific loss function whose covariance operator follows Eqn.~\ref{eq:weighted-cov-op} exactly but with $\theta\equiv 0$:  
\begin{equation}
    L^{\tau,\exact} = \sum_{k=1}^H \mathrm{StopGrad}(\xi^\tau(r_k)) (r_+ - r_{k-}) \label{eq:exact-loss}
\end{equation}
where $r_k = r(\vx, \vx_k)$ is the distance of two unaugmented distinct data points $\vx$ and $\vx_k$ whose augmentation gives $\vx_+$ and $\vx_k$. It is easy to check that $L^{\tau,\exact}$ satisfies the condition $\frac{\partial L}{\partial r_+} + \sum_{k=1}^H \frac{\partial L}{\partial r_{k-}} = 0$. In Sec.~\ref{sec:experiment}, we show that using $L^{\tau,\exact}_{\nce{}}$, the downstream performance in STL-10 is comparable with using regular $\lnce^\tau{}$ loss. This justifies that Theorem~\ref{thm:contrast-simclr-pairwise} indeed separates the most important part of gradient update from less important ones in the residue. 

\textbf{Difference from Neural Tangent Kernels (NTK)}. Note our covariance operator is a completely different mathematical object than the Neural Tangent Kernel (NTK)~\cite{jacot2018neural,arora2019exact}. The NTK is defined in the \emph{sample} space and is full-rank if samples are distinct. For very wide networks (compared to sample size), the NTK seldom changes during training and leads to a convex optimization landscape. On the other hand, the covariance operator $\var_{\vx}\left[\bar K_l(\vx)\right]$ is defined per layer on any data distribution and network of finite width, and does not grow with sample size. 
$\var_{\vx}\left[\bar K_l(\vx)\right]$ may change over the entire training procedure but always remains PSD, leading to many interesting behaviors beyond learning with a fixed kernel. Furthermore, while NTK has nothing to do with data augmentation, $\var_{\vx}\left[\bar K_l(\vx)\right]$ is tied to data augmentation and SSL architectures.  

\def\opintra{\mathrm{EV}}
\def\opinter{\mathrm{VE}}
\def\op{\mathrm{OP}}

\vspace{-0.1in}
\subsection{More general loss functions}
\vspace{-0.1in}
\label{sec:more-general-loss}
For more general loss functions in which $\frac{\partial L}{\partial r_+} + \sum_{k=1}^H \frac{\partial L}{\partial r_{k-}} = \beta \neq 0$, we have a corollary:
\begin{corollary}
SimCLR under $\lsimple^\beta := (1 + \beta) r_+ - r_-$ has the following gradient update rule at layer $l$:
\begin{equation}
    \vec(\Delta W_l)\!=\!\op_l\vec(W_l)\!=\!(-\beta \opintra_l + \opinter_l)\vec(W_l)
\end{equation}
where $\opintra_l$ and $\opinter_l$ are intra-augmentation and inter-augmentation covariance operators at layer $l$:
\begin{eqnarray}
    \opintra_l &:=& \eee{\vx\sim p(\cdot)}{\var_{\vx'\sim p_\aug(\cdot|\vx)}[K_l(\vx')]} \\
    \opinter_l &:=& \var_{\vx\sim p(\cdot)}\left[\eee{\vx'\sim p_\aug(\cdot|\vx)}{K_l(\vx')}\right]
\end{eqnarray}
\end{corollary}
In our experiments, we found that $\beta < 0$ accelerates the training a bit relative to plain SimCLR, possibly due to the fact that $\covop_l$ remains a PSD matrix.

\mathchardef\minus="2D
\def\nodes#1{\mathcal{N}_{#1}}
\def\ch#1{\nodes{#1}^{\mathrm{ch}}}
\def\latents#1{\mathcal{Z}_{#1}}
\def\cardinal#1{m_{#1}}
\def\pr{\mathbb{P}}
\def\cZ{\mathcal{Z}}
\def\diag{\mathrm{diag}}

\def\pred{\mathrm{p}}
\def\sym{\mathrm{s}}
\def\base{\mathrm{b}}
\def\bnterm{\delta W^{\mathrm{\scriptscriptstyle BN}}_l}

\makeatletter
\newcommand{\doublewidetilde}[1]{{%
  \mathpalette\double@widetilde{#1}%
}}
\newcommand{\double@widetilde}[2]{%
  \sbox\z@{$\m@th#1\widetilde{#2}$}%
  \ht\z@=.9\ht\z@
  \widetilde{\box\z@}%
}
\makeatother

\begin{figure*}
    \centering
    \includegraphics[width=0.75\textwidth]{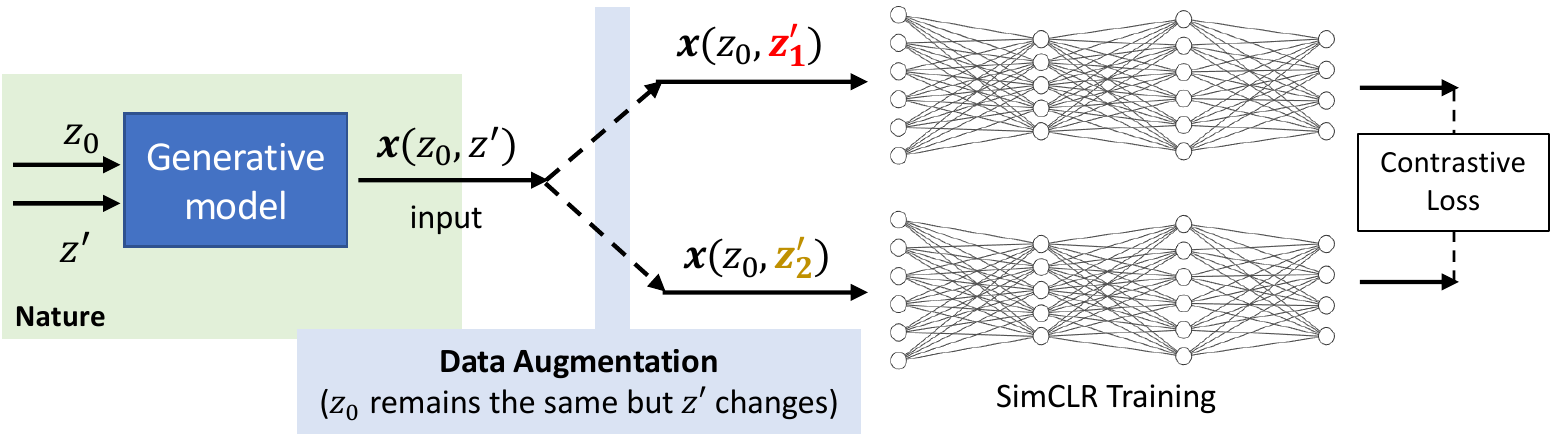}
    \vspace{-0.1in}
    \caption{\small To analyze the functionality of the \emph{covariance operator} $\var_{z_0}\left[\bar K_l(z_0)\right]$ (Theorem~\ref{thm:contrast-simclr-pairwise}), we have Assumption~\ref{assumption:generative-model}: \textbf{(1)} data are generated from some generative model with latent variable $z_0$ and $z'$, \textbf{(2)} augmentation takes $\vx(z_0,z')$, changes $z'$ but keeps $z_0$ intact.}
    \label{fig:overview-assumption}
\end{figure*}

\vspacenoarxiv{-0.1in}
\section{Feature Emergence through Covariance Operator Based Amplification}
\vspacenoarxiv{-0.1in}
\label{sec:feature-emergence}
The covariance operator in Theorem~\ref{thm:contrast-simclr-simple}-\ref{thm:contrast-simclr-pairwise} applies to arbitrary data distributions and augmentations. While this conclusion is general, it is also abstract. To understand what feature representations emerge, we study learning under more specific assumptions on the generative process underlying the data.

\begin{assumption}
\label{assumption:generative-model}
We make two assumptions under the generative paradigm of  (Fig.~\ref{fig:overview-assumption}):
\vspace{-0.1in}
\begin{itemize}
    \item[\textbf{(1)}] The input $\vx=\vx(z_0,z')$ is generated by two groups of latent variables, \emph{class/sample-specific} latents $z_0$ and \emph{nuisance} latents $z'$.
    \item[\textbf{(2)}] Data augmentation changes $z'$ while preserving $z_0$. 
\end{itemize}
\end{assumption}

For brevity of analysis, we use simple loss $\lsimple$ and  Theorem~\ref{thm:contrast-simclr-simple}. Since $\vx = \vx(z_0, z')$, the covariance operator can be represented using expectations over $z_0$ and $z'$: 
\begin{eqnarray}
\covop_l &=& \var_{\vx\sim p(\cdot)}[\eee{\vx'\sim p_\aug(\cdot|\vx)}{K_l(\vx')}] \nonumber \\
&=& \var_{z_0}[\eee{z'|z_0}{K_l(\vx(z_0,z'))}] = \var_{z_0}[\bar K_l(z_0)] \nonumber
\end{eqnarray} 
We leave the analysis of $\lnce{}$ and $\ltriplet{}$ as a future work. At a high-level, they work under similar principles. 

In this setting, we first show that a linear neuron performs dimensionality reduction within an augmentation preserved subspace. We then consider how nonlinear neurons with local receptive fields (RFs) can learn to detect simple objects.  Finally, we extend our analysis to deep ReLU networks exposed to data generated by a hierarchical latent tree model (\hltm), proving that, with sufficient over-parameterization, there exist lucky nodes at initialization whose activation is correlated with latent variables underlying the data, and that SimCLR amplifies these initial lucky representations during learning.

\vspacenoarxiv{-0.05in}
\subsection{SSL and the single neuron: illustrative examples}
\vspacenoarxiv{-0.05in}
\label{sec:one-layer}

\label{sec:non-linear}
\begin{figure}
    \centering
    \includegraphics[width=0.5\textwidth]{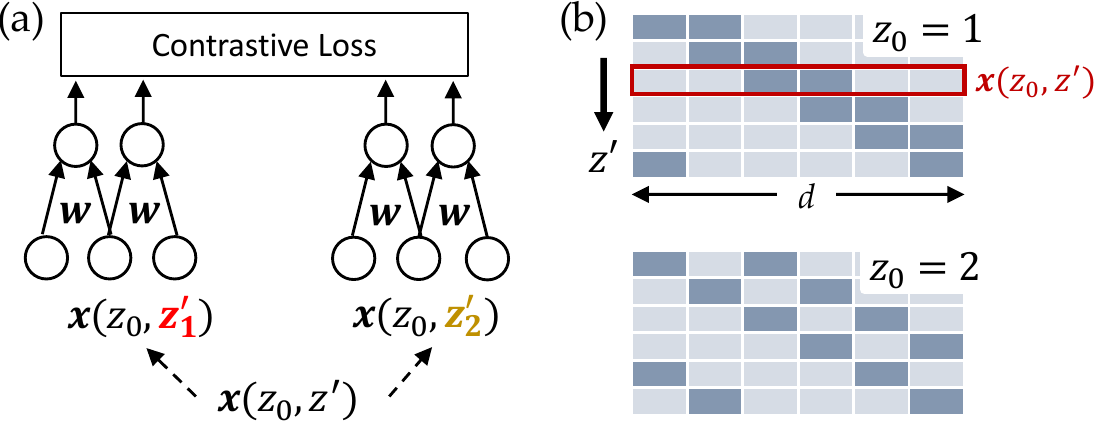}
    \vspace{-0.3in}
    \caption{\small \textbf{(a)} 1-layer convolutional network trained with SimCLR. \textbf{(b)} Its associated generative models: two different objects \texttt{11} ($z_0\!\!=\!\!1$) and \texttt{101} ($z_0\!\!=\!\!2$) undergoes 1D translation. Their locations are specified by $z'$ and subject to change by data augmentation.} 
    \vspace{-0.2in}
    \label{fig:1d-example}
\end{figure}

\paragraph{A single linear neuron performs dimensionality reduction in a subspace preserved by data augmentation.} For a single linear neuron ($L=1, n_L=1$), the connection in Def.~\ref{def:connection} is simply $K_1(\vx) = \vx$. Now imagine that the input space $\vx$ can be decomposed into the direct sum of a semantically relevant subspace, and its orthogonal complement, which corresponds to a subspace of nuisance features.  
Suppose the augmentation distribution $p_\aug(\cdot|\vx)$ is obtained by multiplying $\vx$ by a random Gaussian matrix that acts \emph{only} in the nuisance subspace, thereby identically preserving the semantic subspace. 
Then the augmentation averaged connection $\bar K_1(\vx) = Q^s \vx$ where $Q^s$ is a projection operator onto the semantic subspace. In essence, only the projection of data onto the semantic subspace survives augmentation averaging, as the nuisance subspace is scrambled. Then $\covop = \var_\vx[\bar K_1(\vx)] = Q^s \var_\vx[\vx] Q^{s\t}$.  Thus the covariance of the data distribution, projected onto the semantic subspace, governs the growth of the weight vector $W_1$, demonstrating SimCLR on a single linear neuron performs dimensionality reduction within a semantic subspace preserved by data augmentation.

\textbf{A single linear neuron cannot detect localized objects}.
We now consider a generative model in which data vectors can be thought of as images of objects of the form $\vx(z_0,z')$ where $z_0$ is an important latent semantic variable denoting object identity and $z'$ is a nuisance latent representing its spatial location. The augmentation procedure scrambles position while preserving object identity (Fig.~\ref{fig:1d-example}): 
\begin{equation}
    \vx(z_0,z') = \left\{
        \begin{array}{cc}
            \ve_{z'} + \ve_{(z'+1)\ \mathrm{mod}\ d}& z_0 = 1 \\
            \ve_{z'} + \ve_{(z'+2)\ \mathrm{mod}\ d}& z_0 = 2,
        \end{array}
    \right. \label{eq:translation-generative}
\end{equation}
Specifically, $0\!\le\!z'\!\le\!d-1$ denotes $d$ discrete translational object positions on a periodic ring and
$z_0\in\{1,2\}$ denotes two possible objects \texttt{11} and \texttt{101}. 
The distribution is uniform both over objects and positions: $p(z_0,z') = \frac{1}{2d}$. Augmentation shifts the object to a uniformly random position via $p_\aug(z'|z_0) = 1/d$. For a single linear neuron $K_1(\vx) = \vx$, and the augmentation-averaged connection is $\bar K_1(z_0) = \frac{2}{d}\vone$, and is actually independent of object identity $z_0$ (both objects activate two pixels at any location). Thus $\covop_1 = \var_{z_0}\left[\bar K_1(z_0)\right] = 0$ and no learning happens. 

We next show {\it both} a local RF and nonlinearity can rescue this unfortunate situation.
%
%
\paragraph{A local RF alone does not help.} With the same generative model, now consider a linear neuron with a local RF of width $2$. Within the RF only four patterns can arise: \texttt{00}, \texttt{01}, \texttt{10}, \texttt{11}. Taking the expectation over $z'$ given $z_0$ yields $\bar K_1(z_0\!\!=\!\!1) = \frac{1}{d}\left[\vx_{11} + \vx_{01} + \vx_{10} + (d-3)\vx_{00}\right]$ and $\bar K_1(z_0\!\!=\!\!2) = \frac{1}{d}\left[2\vx_{01} + 2\vx_{10} + (d-4)\vx_{00}\right]$. Here, $\vx_{11} \in \rr^{2}$ denotes pattern \texttt{11}. This yields (here $\vu: = \vx_{11} + \vx_{00} - \vx_{01} - \vx_{10}$): 
\begin{equation}
    \covop_1 = \var_{z_0}\left[\bar K_1(z_0)\right] = \frac{1}{4d^2}\vu\vu^\t \label{eq:cov-local-field}
\end{equation}
and $\covop_1\in\rr^{2\times 2}$ since the RF has width 2. Note that the signed sum of the four pattern vectors in $\vu$ actually cancel, so that $\vu = \vzero$, implying $\covop_1 = 0$ and no learning happens. 

\textbf{A nonlinear neuron with local RF can learn to detect object selective features}. With a ReLU neuron with weight vector $\vw$, from Def.~\ref{def:connection}, the connection is now $K_1(\vx,\vw) = \psi'(\vw^\t \vx) \vx$. Suppose at initialization, $\vw(t)$ happens to be selective for a {\it single} pattern $\vx_p$ (where $p \in \{\texttt{00},\texttt{01},\texttt{10},\texttt{11}\}$), i.e., $\vw(t)^\t \vx_p > 0$ and $\vw(t)^\t \vx_{p'} < 0$ for $p' \neq p$.  The augmentation averaged connection is then $\bar K_1(z_0) \propto \vx_p$ where the proportionality constant depends on object identity $z_0$. Since this averaged connection varies with object identity $z_0$ for all $p$, the covariance operator $\covop_1$ is nonzero and is given by $\var_{z_0}\left[\bar K_1(z_0)\right] = c_p \vx_p\vx_p^\t$ where $c_p>0$ is some constant. By Theorem \ref{thm:contrast-simclr-simple}, the dot product $\vx_p^\t\vw(t)$ grows over time: 
\begin{eqnarray}
    \vx_p^\t\vw(t+1) \!\!&\!=\!&\!\! \vx_p^\t\left(I_{2\times 2} + \alpha c_p \vx_p\vx_p^\t \right) \vw(t) \\
    \!\!&\!=\!&\!\! \left(1 + \alpha c_p\|\vx_p\|^2\right)\vx_p^\t\vw_j(t) > \vx_p^\t\vw_j(t) > 0 \nonumber \label{eq:exclusion}
\end{eqnarray}
Thus the learning dynamics amplifies the initial selectivity to the object selective feature vector $\vx_p$ in a way that cannot be done with a linear neuron. Note this argument also holds with bias terms and initial selectivity for more than one pattern. Moreover,  with a local RF, the probability of weak initial selectivity to some local object sensitive features is high, and we may expect amplification of such weak selectivity in real neural network training, as observed in other settings~\citep{Williams2018-eu}.

\def\vamu{\va_\mu}

\begin{table*}[]
    \centering
    \begin{adjustbox}{width=1\textwidth}
    \small
    \setlength{\tabcolsep}{1.8pt}
    \begin{tabular}{l|l|l|l}
        \textbf{Symbol} & \textbf{Definition} & \textbf{Size} & \textbf{Description}  \\
        \hline
        $\nodes{l}$, $\latents{l}$ & & & The set of all nodes and all latent variables at layer $l$. \\
        $\nodes{\mu}$, $\ch{\mu}$ & &  & Nodes corresponding to latent variable $z_\mu$. $\ch{\mu}$ are children under $\nodes{\mu}$. \\
        \hline
        $P_{\mu\nu}$ & $[\pr(z_\nu|z_\mu)]$ & $2\times 2$ & The top-down transition probability from $z_\mu$ to $z_\nu$. \\ 
        $v_j(z_\mu)$, $\tilde v_j(z_\mu)$ & $\eee{z}{f_j|z_\mu}$, $\mathbb{E}_z[\tilde f_j|z_\mu]$ & scalar, scalar & Expected (pre-)activation $f_j$ (or $\tilde f_j$)  given $z_\mu$  ($z_\mu$'s descendants are marginalized). \\
        $\vf_\mu$, $\vf_{\ch{\mu}}$ & $[f_j]_{j\in\nodes{\mu}}$, $[f_k]_{k\in\ch{\mu}}$ & $|\nodes{\mu}|$, $|\ch{\mu}|$ & Activations for all nodes $j \in \nodes{\mu}$ and for the children of $\nodes{\mu}$ \\ 
        \hline
        $\rho_{\mu\nu}$ & $2 \pr(z_\nu\!\!=\!\!1|z_\mu\!\!=\!\!1) - 1$ & scalar in $[-1,1]$ & Polarity of the transitional probability. \\
        $\rho_0$ & $\pr(z_0=1) - \pr(z_0=0)$ & scalar & Polarity of probability of root latent $z_0$. \\ 
        $s_k$ & $\frac{1}{2}(v_k(1) - v_k(0))$ & scalar & Discrepancy of node $k$ w.r.t its latent variable $z_{\nu(k)}$. \\
        $\vamu$ & $[\rho_{\mu\nu(k)}s_k]_{k\in\ch{\mu}}$ & $|\ch{\mu}|$ & Child selectivity vector. 
    \end{tabular}
    \end{adjustbox}
    \vspace{-0.1in}
    \caption{\small Notation for Sec.~\ref{sec:sb-hltm} (Symmetric Binary \hltm).\label{tab:notation-tree}}
    \vspace{-0.1in}
\end{table*}

\def\vxi{\boldsymbol{\xi}}

\vspacenoarxiv{-0.05in}
\section{Deep ReLU SSL training with Hierarchical Latent Tree Models (\hltm{})}
\vspacenoarxiv{-0.05in}
\label{sec:motivatehltm}

Here we describe a general Hierarchical Latent Tree Model (\hltm{}) of data, and the structure of a multilayer neural network that learns from this data. 

\textbf{Motivation}. The \hltm{} is motivated by the hierarchical structure of our world in which objects may consist of parts, which in turn may consist of subparts. Moreover the parts and subparts may be in different configurations in relation to each other in any given instantiation of the object, or any given subpart may be occluded in some views of an object.  

\textbf{The Generative Model}. The \hltm{} is a very simple toy model that represents a highly abstract mathematical version of such a hierarchical structure.
It consists of a tree structured generative model of data (see Fig.~\ref{fig:hltm}). At the top of the tree (i.e. level $L$), a single categorical latent variable $z_0$ takes one of $m_0$ possible integer values in $\{0\,\dots,m_0-1 \}$, with a prior distribution $\pr(z_0)$.  
One can roughly think of the value of $z_0$ as denoting the identity of one of $m_0$ possible objects.
At level $L-1$ there is a set of latent variables $\latents{L-1}$. 
This set is indexed by $\mu$ and each latent variable $z_\mu$ is itself a categorical variable that takes one of $\cardinal\mu$ values in $\{0,\dots,\cardinal\mu-1\}$. Roughly we can think of each latent variable $z_\mu$ as corresponding to a part, and the different values of $z_\mu$ reflect different configurations or occlusion states of that part.  
The transition matrix (conditional probability) $\pr(z_\mu|z_0)\in \rr^{m_0 \times m_\mu}$ can roughly be thought of as collectively reflecting the distribution over the presence or absence, as well as configurational and occlusional states of each part $\mu$, conditioned on object identity $z_0$.  This process can continue onwards to describe subparts of parts, until layer $l=0$ (the ``pixel'' level). All the latent variables at $l = 0$ are now visible and form a signal $\vx(z_0, z')$ received by the deep network for SSL training. 

\textbf{Data Augmentation}. Given a sample $\vx=\vx(z_0,z')$, data augmentation involves \emph{resampling} all $z_\mu$ (which are $z'$ in Fig.~\ref{fig:hltm}), while fixing the root $z_0$. This models augmentations as changing part configurations while keeping object identity $z_0$ fixed. 

\begin{figure}[t]
    \centering
    \includegraphics[width=0.5\textwidth]{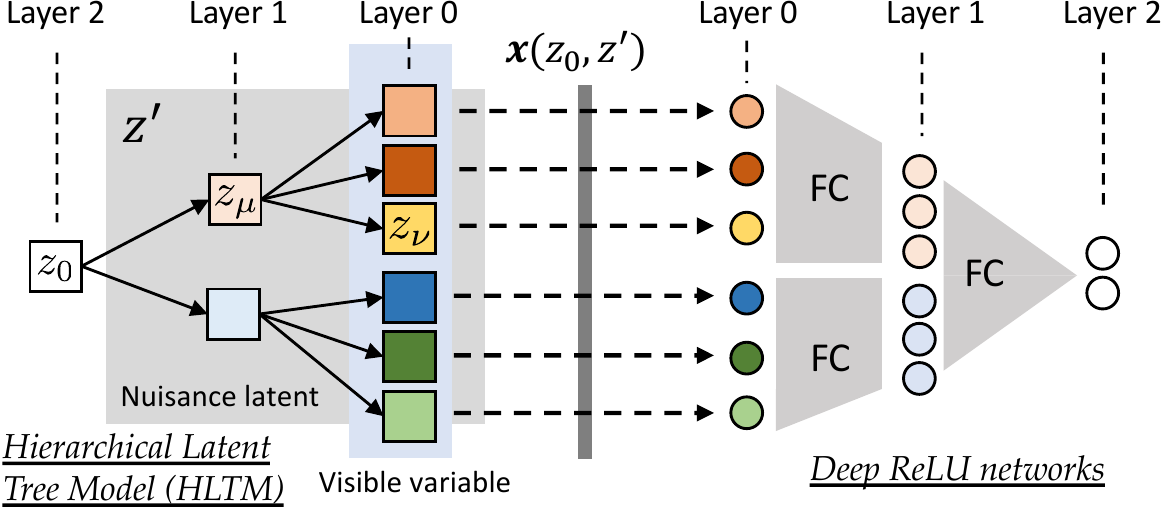}
    \vspace{-0.3in}
    \caption{\small Hierarchical Latent Tree Models. A latent variable $z_\mu$, and its corresponding nodes $\nodes{\mu}$ in multi-layer ReLU side, covers a subset of input $\vx$, resembling local receptive fields in ConvNet.}
    \label{fig:hltm}
    \vspace{-0.2in}
\end{figure}

\textbf{The neural network.} We now consider the multi-layer ReLU network that learns from data generated from \hltm{} (right hand side of Fig.~\ref{fig:hltm}). For simplicity let $L=2$. The neural network has a set of input neurons that are in one to one correspondence with the pixels or visible variables $z_\nu$ that arise at the leaves of the \hltm{}, where $l=0$. For any given object $z_0$ at layer $l = 2$, and its associated parts states $z_\mu$ at layer $l=1$, and visible feature values $z_\nu$ at layer $l=0$, the input neurons of the neural network receive {\it only} the visible feature values $z_\nu$ as real analog inputs.  Thus the neural network \emph{does not have direct access to} the latent variables $z_0$ and $z_\mu$ that generate these visible variables.

\textbf{The objective}. Under this setting, one key question is: \emph{what do the hidden units of the neural network learn?} In particular, can the network learn some hidden unit $j$ whose activation $f_j$ correlates well with the value of a latent variable $z_\mu$? Or more precisely, does $\ee{f_j|z_\mu}$ correlate strongly with $z_\mu$ even if $j$ \emph{never receives any direct supervision} from $z_\mu$ during SSL training? In this paper, we make a first attempt to address this question in a simplified setting.

\subsection{Symmetric Binary \hltm{} (SB-\hltm{})}
\label{sec:sb-hltm}
To ease the analysis, we consider a simpler version of \hltm{}: symmetric binary \hltm{}. At layer $l$, we have latent \emph{binary} variables $\{z_\mu\}$, where $\mu\in\latents{l}$ indexes different latent variables and each $z_\mu\in \{0,1\}$. The topmost latent variable is $z_0$. Following the tree structure, for $\mu\in \latents{l}$ and $\nu_1,\nu_2\in\latents{l-1}$, conditional independence holds: $\pr(z_{\nu_1},z_{\nu_2}|z_\mu)\!\!=\!\!\pr(z_{\nu_1}|z_\mu) \pr(z_{\nu_2}|z_\mu)$. For $\pr(z_\nu|z_\mu)$, we assume it is \emph{symmetric}: for $\mu\in\latents{l}$ and $\nu\in\latents{l-1}$:
\begin{equation}
    \pr(z_\nu\!\!=\!\!1|z_\mu\!\!=\!\!1)= \pr(z_\nu\!\!=\!\!0|z_\mu\!\!=\!\!0)=(1+\rho_{\mu\nu})/2
\end{equation}
where the \emph{polarity} $\rho_{\mu\nu} \in [-1,1]$ measures how informative $z_\mu$ is. If $\rho_{\mu\nu} = \pm 1$ then there is no stochasticity in the top-down generation process. If  $\rho_{\mu\nu}=0$, then there is no information in the downstream latents and the posterior of $z_0$ given the observation $\vx$ can only be uniform. See Appendix for more general cases. 

The final sample $\vx$ is a collection of all visible leaf variables (Fig.~\ref{fig:hltm}), and thus depends on all latent variables.  
Corresponding to the hierarchical tree model, each neural network node $j\in\nodes{l}$ maps to a unique $\mu = \mu(j) \in \latents{l}$. Let $\nodes{\mu}$ be all nodes that map to $\mu$. While in the pixel level, there is a 1-1 correspondence between the children $\nu$ of a subpart $\mu$ and the pixel, in the hidden layer, more than one neuron could correspond to $z_\mu$ and thus $|N_\mu|>1$, which is a form of \emph{over-parameterization}. We further let $\ch{\mu}$ denote the subset of nodes that provide input to nodes in $\nodes{\mu}$. For $j\in \nodes{\mu}$, its activation $f_j$ only depends on the value of $z_\mu$ and its descendant latent variables, through input $\vx$. Define $v_j(z_\mu) := \eee{z}{f_j|z_\mu}$ as the expected activation conditioned on $z_\mu$. See Tbl.~\ref{tab:notation-tree} for a summary of symbols.

\subsubsection{Lucky nodes at initialization}
\label{sec:lucky-node-main-text}
We mainly study the following question: for a latent variable $z_\mu$ at some intermediate layer, is there a node $j$ in the deep ReLU network at the corresponding layer that correlates strongly with $z_\mu$? In binary \hltm{}, this is means asking whether the \emph{selectivity} $s_j := (v_j(1) - v_j(0)) / 2$ is high after training.
If $|s_j|$ is large (or \emph{highly selective}), then the node $j$ changes its behavior drastically for $z_\mu = 0\ \mathrm{or}\ 1$, and thus $j$ is highly correlated with the latent variable $z_\mu$. We show this arises over training in two steps. 

First, we prove that given sufficient \emph{over-parameterization} ($|\nodes{\mu}| \gg 1$), even at initialization, without any training, we can find some lucky nodes with weak  selectivity:
\begin{theorem}[Theorem Sketch, Lucky node at initialization for SB-\hltm{}]
\label{thm:lucky-node}
When random weight initialization and $|\nodes{\mu}| = \mathcal{O}(c\cdot e^{c/2}\log 1 / \eta)$, with probability at least $1 -\eta$, there exists at least one node $j \in\nodes{\mu}$ so that the \emph{pre-activation gap} $\tilde v_j(1) - \tilde v_j(0) = 2\vw_j^\t\vamu > 0$ and its selectivity $|s_j| \ge \phi(\rho^2_{\mu\nu}, \{s_k\}_{k\in \ch{\mu}}, c)$.
\end{theorem}
See Appendix for detailed theorem description and proof. $\vamu$ is defined in Fig.~\ref{fig:hltm-further-notation}(c) and $\phi$ is a weak threshold that increases monotonically w.r.t. all its arguments. This means that higher polarity, more selectivity in the lower layer and more over-parameterization (larger $|\nodes{\mu}|$) all boost weak initial selectivity of a lucky node.  

\begin{figure}
    \centering
    \includegraphics[width=0.5\textwidth]{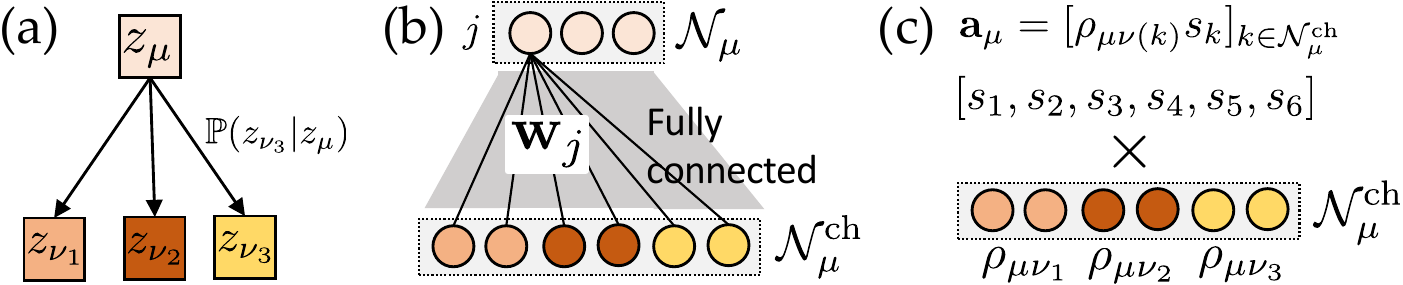}
    \vspace{-0.3in}
    \caption{\small Notation used in Theorem~\ref{thm:lucky-node} and Theorem~\ref{thm:cov-operator}. \textbf{(a)} Latent variable structure. \textbf{(b)} A fully connected part of \hltm{}. Conceptually, after SSL training, nodes (in circle) should realize the latent variables (in square) of the same color, while they never receive direct supervision from them.
    $\vw_j$ is a weight vector that connect top node $j\in \nodes{\mu}$ to all nodes in $\ch{\mu}$. For this FC part, we can also compute a covariance operator $\covop_\mu$ and Jacobian $J_\mu$. \textbf{(c)} $\vamu := [\rho_{\mu\nu(k)}s_k]_{k\in \ch{\mu}}$ is element-wise product between selectivity and polarity of all child nodes of $z_\mu$. Its length is $|\ch{\mu}|$.}
    \label{fig:hltm-further-notation}
\end{figure}

\subsubsection{Training with constant Jacobian}
\label{sec:hltm-training}
Second, we show training strengthens this weak initial selectivity. We compute covariance operator $\covop_\mu =\var_{z_0}[\bar K_\mu(z_0)]$ at different fully-connected part (Fig.~\ref{fig:hltm-further-notation}(b)), indexed by latent variable $z_\mu$. Here $\bar K_\mu(z_0) = \eee{z'}{\vf_{\ch{\mu}} \otimes J_\mu^\t|z_0}$. Here we assume $J_\mu$ is a constant matrix and mainly check the term $\eee{z'}{\vf_{\ch{\mu}}|z_0}$, which turns out to have a nice close form. 
\begin{theorem}[Activation covariance in SB-\hltm]
\label{thm:cov-operator}
$\var_{z_0}[\eee{z'}{\vf_{\ch{\mu}}|z_0}] = o_\mu\vamu\vamu^\t$. Here $o_\mu := \rho^2_{0\mu} (1-\rho_0^2)$.
\end{theorem}
Here $\rho_{0\mu}$ is the polarity between $z_0$ and $z_\mu$, which might be far apart in hierarchy. A simple computation (See Lemma and associated remarks in Appendix) shows that $\rho_{0\mu} = \prod_{0,\ldots,\alpha,\beta,\ldots,\mu}\rho_{\alpha\beta}$ is a product of consequent polarities in the tree hierarchy. 


Theorem~\ref{thm:cov-operator} suggests when $\rho_{0\mu}$ and $\|\vamu\|$ are large, the covariance $\covop_\mu = o_\mu\vamu\vamu^\t \otimes J_\mu^\t J_\mu$ has large magnitude and training is faster. Note that $\|\vamu\|$ is large when the children are highly selective (large $|s_k|$) and/or the magnitude of the polarity $|\rho_{\mu\nu}|$ is large (i.e., the top-down generation process is more deterministic). 

Note that if $\max_{\alpha\beta} |\rho_{\alpha\beta}| < 1$, then $\lim_{L\rightarrow+\infty} \rho_{0\mu}\rightarrow 0$, i.e., polarity $\rho_{0\mu}$ vanishes for very deep latent tree models due to mixing of the Markov Chain. In this case, $P_{0\mu}$ becomes uniform, making $\covop_\mu$ small. Thus training in SSL is faster at the top layers where the covariance operators have larger magnitude.  

If we further assume $J_\mu^\t J_\mu = I$, then after the gradient update, for the ``lucky'' node $j$ we have:
\begin{eqnarray}
\vspacenoarxiv{-0.05in}
\vamu^\t\vw_j(t + 1) \!\!&\!=\!&\!\! \vamu^\t\left[I + \alpha o_\mu\vamu\vamu^\t\right]\vw_j(t) \\
\!\!&\!=\!&\!\! (1+\alpha o_\mu\| \vamu\|^2_2) \vamu^\t\vw_j(t) >  \vamu^\t\vw_j(t) > 0 \nonumber
\end{eqnarray}
which means that the pre-activation gap $\tilde v_j(1) - \tilde v_j(0) = 2\vw_j^\t\vamu$ grows over time and the latent variable $z_\mu$ is \emph{learned} (instantiated as $f_j$) during training, even if the network is never supervised with its true value. While in practice $J_\mu$ changes over time, here we give a simple demonstration and leave detailed analysis for future work. 

In Sec.~\ref{sec:experiment}, as predicted by our theory, the intermediate layers of deep ReLU networks do learn the latent variables of the \hltm\ (see Tbl.~\ref{tbl:hltm} below and Appendix). 

\vspacenoarxiv{-0.10in}
\section{Experiments}
\vspacenoarxiv{-0.10in}
\label{sec:experiment}
We test our theoretical findings through experiments on CIFAR-10~\cite{krizhevsky2009learning} and STL-10~\cite{coates2011analysis}. We use a simplified linear evaluation protocol: the linear classifier is trained on frozen representations computed \emph{without} data augmentation. This reuses pre-computed representations and is more efficient. We use ResNet-18 as the backbone and all experiments are repeated 5 times for mean and std. Please check detailed setup in Appendix.

\textbf{Verification of Theorem~\ref{thm:contrast-simclr-pairwise}}. One question is whether the residue term actually plays a major role in the gradient update. We verify the dominant role of the covariance operator covariance operator over the residue term, by designing a specific weighted $\ell_2$ loss function ($\lnce^{\tau,\exact{}}$ in Eqn.~\ref{eq:exact-loss}) that yields the identical covariance operator as in Theorem~\ref{thm:contrast-simclr-pairwise} but has no residue term. Tbl.~\ref{tbl:exact-versus-lnce} shows the performance is comparable with a normal InfoNCE loss. Furthermore, if we add noise ($=\eta \cdot \mathrm{XavierInit}(W_l)$) to the gradient update rule, the performance improves slightly. 

\begin{table}[t]
\centering
\small
\setlength{\tabcolsep}{1pt}
\caption{Comparison between $L^{\tau,\exact}_{\lnce{}}$ and $\lnce^\tau$. Top-1 accuracy with linear evaluation protocol. $\tau=0.5$.\label{tbl:exact-versus-lnce}}
\begin{adjustbox}{width=1\linewidth}
\begin{tabular}{l||c|c|c}
         & 100 epochs & 300 epochs & 500 epochs \\ 
\hline
&\multicolumn{3}{c}{\emph{CIFAR-10}} \\
\hline 
$\lnce^{\tau,\exact}$ & $83.84\pm 0.18$ & $87.49 \pm 0.32$  &  $87.65\pm 0.34$ \\
$\lnce^{\tau,\exact}$ ($\eta=0.01$) & $84.04\pm 0.16$ & $\mathbf{88.23 \pm 0.09}$  & $\mathbf{88.82\pm 0.25}$ \\
$\lnce^\tau$ & $\mathbf{84.20\pm 0.09}$ & $87.57\pm 0.31$ & $87.81\pm 0.37$ \\
\hline
&\multicolumn{3}{c}{\emph{STL-10}} \\
\hline 
$\lnce^{\tau,\exact}$ & $78.62\pm 0.25$ & $82.57\pm 0.18$ & $83.59\pm 0.14$\\ 
$\lnce^{\tau,\exact}$ ($\eta=0.01$) & $78.27\pm 0.25$ & $82.33 \pm 0.24$  &  $83.72\pm 0.16$ \\
$\lnce^\tau$ & $\mathbf{78.82\pm 0.10}$ & $\mathbf{82.68\pm 0.20}$ & $\mathbf{83.82\pm 0.11}$ \\ \hline
\end{tabular}
\end{adjustbox}
\vspace{-0.2in}
\end{table}
%

\textbf{Extended contrastive loss function.} When $\frac{\partial L}{\partial r_+} + \frac{\partial L}{\partial r_-}{=}\beta{\neq}0$, the covariance operator can still be derived (Sec.~\ref{sec:more-general-loss}) and remains PSD when $\beta{<}0$.
As shown in Tbl.~\ref{tab:negative-beta}, we find that (1) $\beta{<}0$ performs better at first 100 epochs but converges to similar performance after 500 epochs, suggesting that $\beta{<}0$ might accelerate training, (2) $\beta{>}0$ worsens performance. We report a similar observation on ImageNet~\cite{imagenet_cvpr09}, where our default SimCLR implementation achieves 64.6\% top-1 accuracy with 60-epoch training; setting $\beta = -0.5$ yields 64.8\%; and setting $\beta{>}0$ hurts the performance.

\begin{table}[t]
    \centering
    \small
    \setlength{\tabcolsep}{3pt}
    \caption{Extended contrastive loss function with $\beta \neq 0$ (Sec.~\ref{sec:more-general-loss}).\label{tab:negative-beta}}
    \begin{tabular}{c||c|c|c}
                & 100 epochs & 300 epochs & 500 epochs \\ 
    \hline
    $\beta$       &   \multicolumn{3}{c}{\emph{CIFAR-10}} \\
    \hline
    $-0.5$    & $85.17\pm 0.36$ & $88.00\pm 0.29$ & $88.14\pm 0.27$ \\
    $+0.2$    & $82.87\pm 0.08$  &  $87.16\pm 0.15$  & $87.29\pm 0.07$ \\ 
    \hline
    $\beta$       &   \multicolumn{3}{c}{\emph{STL-10}} \\
    \hline
    $-0.5$     & $79.72\pm 0.09$ & $82.80\pm 0.23$ & $83.75\pm 0.16$ \\
    $+0.2$      & $77.48\pm 0.29$  &  $82.15\pm 0.02$  & $83.33\pm 0.19$ \\
    \hline
    \end{tabular}
    \vspace{-0.2in}
\end{table}

\textbf{Hierarchical Latent Tree Model (\hltm)}. We implement \hltm{} and check whether the intermediate layers of deep ReLU networks learn the corresponding latent variables at the same layer. The degree of learning is measured by the normalized correlations between the ground truth latent variable $z_\mu$ and its best corresponding node $j\in \nodes{\mu}$.  Tbl.~\ref{tbl:factors-in-bn-simclr} indicates this measure increases with over-parameterization and learning, consistent with our analysis (Sec.~\ref{sec:sb-hltm}). More experiments in Appendix.

\begin{table}[h!]
\centering
\small
\label{tbl:hltm}
\caption{\small Normalized Correlation between the topmost latent variable (most difficult to learn) in SB-\hltm{} and topmost nodes in deep ReLU networks ($L = 5$) trained with SimCLR with NCE loss. With more over-parameterization, correlations are higher with lower std on 10 trials at both init and end of training, }
\label{tbl:factors-in-bn-simclr}
\setlength{\tabcolsep}{1pt}
\begin{tabular}{c||c|c||c|c}
    \multirow{2}{*}{$\rho_{\mu\nu}$} & \multicolumn{2}{c||}{$|\nodes{\mu}|=2$} & \multicolumn{2}{c}{$|\nodes{\mu}|=10$} \\
      & Initial & Converged & Initial & Converged \\
    \hline
    $\sim \mathrm{U}[0.7,1]$ & $0.35\pm 0.20$ & $0.62\pm 0.29$ & $0.62\pm  0.10$ & $0.85\pm  0.05$ \\ 
    $\sim \mathrm{U}[0.8,1]$ & $0.48\pm 0.23$ & $0.72\pm 0.31$ & $0.75\pm  0.08$ & $0.91\pm  0.03$ \\
    $\sim \mathrm{U}[0.9,1]$ & $0.66\pm 0.28$ & $0.80\pm 0.29$ & $0.88\pm  0.05$ & $0.96\pm  0.01$
\end{tabular}
\vspace{-0.2in}
\end{table}

\section{Conclusion and Future Works}

In this paper, we propose a novel theoretical framework to study self-supervised learning (SSL) paradigms that consist of dual deep ReLU networks. We analytically show that the weight update at each intermediate layer is governed by a covariance operator, a PSD matrix that amplifies weight directions that align with variations across data points which survive averages over augmentations. We show how the operator interacts with multiple generative models that generate the input data distribution, including a simple 1D model with circular translation and hierarchical latent tree models. Experiments support our theoretical findings. 

To our best knowledge, our work is the first to open the blackbox of deep ReLU neural networks to bridge contrastive learning, (hierarchical) generative models, augmentation procedures, and the emergence of features and representations. We hope this work opens new opportunities and perspectives for the research community. 

\bibliographystyle{icml2021}
\bibliography{references}

\clearpage 
\appendix

\setcounter{lemma}{0}
\setcounter{corollary}{0}
\setcounter{theorem}{0}

\onecolumn


\begin{figure}[h!]
    \centering
    \includegraphics[width=\textwidth]{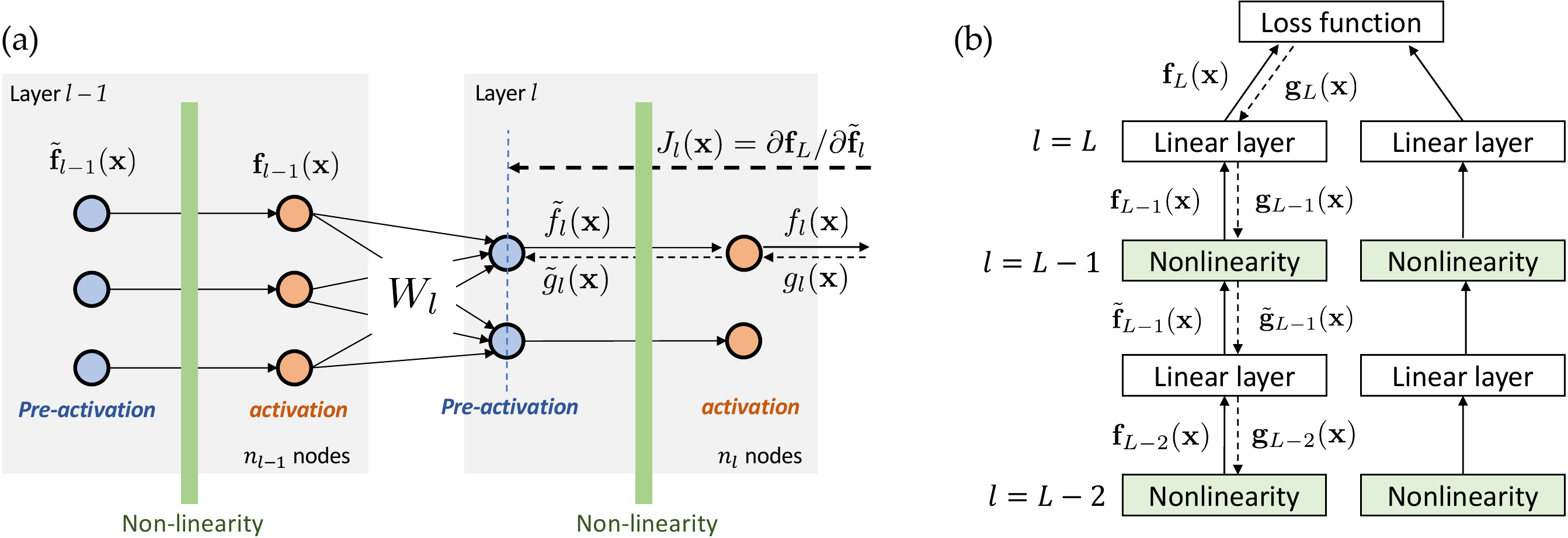}
    \caption{Notation and basic setting. \textbf{(a)} Definition of activation $\vf_l$, pre-activation $\tilde \vf_l$, backpropagated gradient before nonlinearity $\vg_l$, backpropagated gradient after nonlinearity $\tilde\vg_l$, and Jacobian $J_l(\vx) := \partial \vf_L / \partial \tilde\vf_l$. \textbf{(b)} Dual network structure.}
    \label{fig:appendix-basic-setting}
\end{figure}

\section{Background and Basic Setting (Section~\ref{sec:background})}
\subsection{Lemmas}

\def\f{\mathrm{f}}
\def\b{\mathrm{b}}
\def\n{\mathrm{n}}

\begin{definition}[\revnoun{}]
\label{def:reversibility}
A layer $l$ is \emph{\rev{}} if there is a $G_l(\vx;\cW) \in \rr^{n_l\times n_{l-1}}$ so that the \emph{pre-activation} at layer $l$ satisfies $\tilde \vf_l(\vx;\cW) = G_l(\vx;\cW)\tilde\vf_{l-1}(\vx;\cW)$ and backpropagated gradient \emph{after nonlinearity} $\tilde \vg_{l-1} = G^\t_l(\vx;\cW) Q^\t_l(\vx;\cW)\tilde\vg_l$ for some matrix $Q_l(\vx;\cW)\in \rr^{n_l\times n_l}$. A network is \emph{\rev{}} if all layers are.
\end{definition}
Note that many different kinds of layers have this reversible property, including linear layers (MLP and Conv) and (leaky) ReLU nonlinearity. For linear layers, at layer $l$, we have:
\begin{equation}
G_l(\vx;\cW) = W_l, \quad\quad Q_l(\vx;\cW) \equiv I_{n_l\times n_l}
\end{equation}
For multi-layer ReLU network, for each layer $l$, we have:
\begin{equation}
G_l(\vx;\cW) = W_l D_{l-1}(\vx;\cW), \quad\quad Q_l(\vx;\cW) \equiv I_{n_l\times n_l} \label{eq:multi-layer-relu}
\end{equation} 
where $D_{l-1} \in \rr^{n_{l-1}\times n_{l-1}}$ is a binary diagonal matrix that encodes the gating of each neuron at layer $l-1$. The gating $D_{l-1}(\vx;\cW)$ depends on the current input $\vx$ and current weight $\cW$. 

In addition to ReLU, other activation function also satisfies this condition, including linear, LeakyReLU and monomial activations. For example, for power activation $\psi(x) = x^p$ where $p > 1$, we have:
\begin{equation}
    G_l(\vx; \cW) = W_l \diag^{p-1}(\tilde\vf_{l-1}), \quad\quad Q_l(\vx;\cW) \equiv p I_{n_l\times n_l} \label{eq:power-activation}
\end{equation}

\textbf{Remark}. Note that the \revnoun{} is not the same as invertible. Specifically, \revnoun{} only requires the transfer function of a backpropagation gradient is a transpose of the forward function. 

\begin{lemma}[Recursive Gradient Update (Extension to Lemma 1 in~\cite{tian2019student}]
\label{lemma:extension-multiplicative}
Let forward and backward \emph{transition} matrix $V^\f_L(\vx) = V^\b_L(\vx) = I_{n_L\times n_L}$, and define recursively:
\begin{eqnarray}
    V^\f_{l-1}(\vx) &:=& V^\f_l(\vx) G_l(\vx)\in \rr^{n_l\times n_{l-1}} \\
    V^\b_{l-1}(\vx) &:=& V^\b_l(\vx) Q_l(\vx) G_l(\vx)\in \rr^{n_l\times n_{l-1}}
\end{eqnarray}
If the network is \rev{} (Def.~\ref{def:reversibility}), then minimizing the $\ell_2$ objective:
\begin{equation}
    r(\cW_1) := \frac{1}{2}\|\vf_L(\vx_1;\cW_1) - \vf_L(\vx_2;\cW_2)\|^2_2 
\end{equation}
with respect to weight matrix $W_l$ at layer $l$ yields the following gradient at layer $l$:
\begin{equation}
    \tilde\vg_l = V_l^{\b\t}(\vx_1;\cW_1)\left[V^\f_l(\vx_1;\cW_1)\tilde\vf_l(\vx_1;\cW_1) - V^\f_l(\vx_2;\cW_2)\tilde\vf_l(\vx_2;\cW_2)\right]
\end{equation}
\end{lemma}
\begin{proof}
We prove by induction. Note that our definition of $W_l$ is the transpose of $W_l$ defined in~\cite{tian2019student}. Also our $\vg_l(\vx)$ is the gradient \emph{before} nonlinearity, while~\cite{tian2019student} uses the same symbol for the gradient after nonlinearity. 

For notation brievity, we let $\tilde\vf_l(\vx_1) := \tilde\vf_l(\vx_1;\cW_l)$ and $G_l(\vx_1) := G_l(\vx_1;\cW_l)$. Similar for $\vx_2$ and $W_2$.

When $l = L$, by the property of $\ell_2$-loss and the fact that $\vf_L = \tilde \vf_L$ (no nonlinearity in the top-most layer), we know that $\vg_L = \tilde\vf_L(\vx_1;\cW_1) - \tilde\vf_L(\vx_2;\cW_2)$, by setting $V^\f_L(\vx_1) = V^\f_L(\vx_2) = V^\b_L(\vx_1) = V^\b_L(\vx_2) = I$, the condition holds. Now suppose for layer $l$, we have:
\begin{eqnarray}
\tilde\vg_l = V_l^{\b\t}(\vx_1)\left[V^\f_l(\vx_1)\tilde\vf_l(\vx_1) - V^\f_l(\vx_2)\tilde\vf_l(\vx_2)\right]
\end{eqnarray}

Then:
\begin{eqnarray}
\tilde\vg_{l-1} &=& G^\t_l(\vx_1) Q^\t_l(\vx_1) \tilde\vg_l \\
&=& \underbrace{G^\t_l(\vx_1) Q^\t_l(\vx_1) V_l^{\b\t}(\vx_1)}_{V_{l-1}^{\b\t}(\vx_1)} \cdot \left[V^\f_l(\vx_1)\tilde\vf_l(\vx_1) - V^\f_l(\vx_2)\tilde\vf_l(\vx_2) \right] \\
&=& V_{l-1}^{\b\t}(\vx_1)\left[\underbrace{ V^\f_l(\vx_1) G_l(\vx_1)}_{V^\f_{l-1}(\vx_1)}\tilde\vf_{l-1}(\vx_1) - \underbrace{V^\f_l(\vx_2)G_l(\vx_2)}_{V^\f_{l-1}(\vx_2)}\tilde\vf_{l-1}(\vx_2)\right] \\
&=& V_{l-1}^{\b\t}(\vx_1)\left[V^\f_{l-1}(\vx_1)\tilde\vf_{l-1}(\vx_1) - V^\f_{l-1}(\vx_2)\tilde\vf_{l-1}(\vx_2)\right]
\end{eqnarray}
\end{proof}
\textbf{Remark on Deep ReLU networks}. Note that for multi-layered ReLU network, $G_l(\vx) = D_l(\vx) W_l$, $Q_l(\vx) \equiv I$ for each ReLU+Linear layer. Therefore, for each layer $l$, we have $V_l^\f(\vx) = V_l^\b(\vx)$ and we can just use $V_l(\vx) := V_l^\f(\vx) = V_l^\b(\vx)$ to represent both. If we set $\vx_1=\vx_2=\vx$, $\cW_1 = \cW$, $\cW_2 = \cW^*$ (teacher weights), then we go back to the original Lemma~1 in~\cite{tian2019student}.

\textbf{Remark on $\ell_2$-normalization in the topmost layer}. For $\ell_2$-normalized objective function on Deep ReLU networks:
\begin{equation}
    r_{\n}(\cW_1) := \frac{1}{2}\Bigg\|\frac{\vf_L(\vx_1;\cW_1)}{\|\vf_L(\vx_1;\cW_1)\|_2} - \frac{\vf_L(\vx_2;\cW_2)}{\|\vf_L(\vx_2;\cW_2)\|_2}\Bigg\|^2_2, \label{eq:loss-l2-normalization}
\end{equation}
it is equivalent to add a top-most $\ell_2$-normalization layer $\vf_{L+1} := \frac{\vf_L}{\|\vf_L\|_2}$, we have $G_{L+1} := \frac{1}{\|\vf_L\|_2} I_{n_L\times n_L}$ and due to the following identity (here $\hat\vy := \vy / \|\vy\|_2$): 
\begin{equation}
    \frac{\partial \hat\vy}{\partial \vy} = \frac{1}{\|\vy\|_2}\left(I - \hat\vy \hat\vy^\t\right)  
\end{equation}
Therefore we have $\partial \vf_{L+1} / \partial \vf_L = (I_{n_L\times n_L}-\vf_{L+1}\vf_{L+1}^\t)G_{L+1}$ and we could set $Q_{L+1} := I-\vf_{L+1}\vf_{L+1}^\t$, which is a projection matrix (and is also a positive semi-definite matrix). Let $\ell_2$-normalized transition matrix $V_l^\n(\vx) :=  \frac{1}{\|\vf_L(\vx)\|_2} V_l(\vx)$. Applying Lemma~\ref{lemma:extension-multiplicative} and we have for $1\le l \le L$:
\begin{equation}
    \tilde\vg_l = V_l^{\n\t}(\vx_1;\cW_1) Q_{L+1}(\vx_1;\cW_1)\left[V^\n_l(\vx_1;\cW_1)\tilde\vf_l(\vx_1;\cW_1) - V^\n_l(\vx_2;\cW_2)\tilde\vf_l(\vx_2;\cW_2)\right]
\end{equation}

\textbf{Remark on ResNet.} Note that the same structure holds for blocks of ResNet with ReLU activation.

\subsection{Theorem~\ref{thm:l2-two-tower}}
Now we prove Theorem~\ref{thm:l2-two-tower}. Note that for deep ReLU networks, $Q_l$ is a simple identity matrix and thus:

\begin{equation}
J_l(\vx) := \frac{\partial \vf_L}{\partial \tilde \vf_l} = V_l(\vx) := V_l^\f(\vx) = V_l^\b(\vx) \label{eq:j-v}
\end{equation}

\begin{theorem}[Squared $\ell_2$ Gradient for dual deep reversible networks]
The gradient $g_{W_l}$ of the squared loss $r$ with respect to $W_l \in \rr^{n_l \times n_{l-1}}$ for a single input pair $\{\vx_1, \vx_2\}$ is: 
\begin{equation}
    g_{W_l} = \vec\left(\partial r/\partial W_{1,l}\right) = K_{1,l}\left[K^\t_{1,l}\vec(W_{1,l}) - K^\t_{2,l}\vec(W_{2,l})\right].
\end{equation}
Here $K_l(\vx;\cW) := \vf_{l-1}(\vx;\cW) \otimes J^\t_l(\vx;\cW)$, $K_{1,l} := K_l(\vx_1;\cW_1)$ and $K_{2,l} := K_l(\vx_2;\cW_2)$.
\end{theorem}
\begin{proof}
We consider a more general case where the two towers have different parameters, namely $\cW_1$ and $\cW_2$. Applying Lemma~\ref{lemma:extension-multiplicative} for the branch with input $\vx_1$ at the linear layer $l$, and using Eqn.~\ref{eq:j-v} we have:
\begin{equation}
    \tilde \vg_{1,l} = J_{1,l}^\t[J_{1,l}W_{1,l}\vf_{1,l-1} - J_{2,l}W_{2,l}\vf_{2,l-1}] 
\end{equation}
where $\vf_{1,l-1} := \vf_{l-1}(\vx_1;\cW_1)$ is the activation of layer $l-1$ just below the linear layer at tower 1 (similar for other symbols), and $\tilde \vg_{1,l}$ is the back-propagated gradient \emph{after} the nonlinearity. 

In this case, the gradient (and the weight update, according to gradient descent) of the weight $W_l$ between layer $l$ and layer $l-1$ is: 
\begin{eqnarray}
    \frac{\partial r}{\partial W_{1,l}} &=& \tilde\vg_{1,l}\vf^\t_{1,l-1} \\
    &=& J^\t_{1,l} J_{1,l}W_{1,l}\vf_{1,l-1}\vf^\t_{1,l-1} - J^\t_{1,l} J_{2,l}W_{2,l}\vf_{2,l-1}\vf^\t_{1,l-1}
\end{eqnarray}
Using $\vec(AXB) = (B^\t\otimes A)\vec(X)$ (where $\otimes$ is the Kronecker product), we have:
\begin{equation}
    \vec\left(\frac{\partial r}{\partial W_{1,l}}\right) = \left(\vf_{1,l-1}\vf^\t_{1,l-1} \otimes J^\t_{1,l} J_{1,l}\right)\vec(W_{1,l}) - \left(\vf_{1,l-1}\vf^\t_{2,l-1} \otimes J^\t_{1,l} J_{2,l}\right)\vec(W_{2,l})
\end{equation}
Let 
\begin{equation}
    K_l(\vx; W) := \vf_{l-1}(\vx;\cW)\otimes J_l^\t(\vx;\cW) \in \rr^{n_l n_{l-1} \times n_L}    
\end{equation}
Note that $K_l(\vx ;\cW)$ is a function of the current weight $W$, which includes weights at all layers. By the mixed-product property of Kronecker product $(A\otimes B)(C\otimes D) = AC\otimes BD$, we have:
\begin{eqnarray}
    \vec\left(\frac{\partial r}{\partial W_{1,l}}\right) &=& K_l(\vx_1)K_l(\vx_1)^\t\vec(W_{1,l}) - K_l(\vx_1)K_l(\vx_2)^\t\vec(W_{2,l}) \\
    &=& K_l(\vx_1)\left[K_l(\vx_1)^\t\vec(W_{1,l}) - K_l(\vx_2)^\t\vec(W_{2,l})\right]
\end{eqnarray}
where $K_l(\vx_1) = K_l(\vx_1; \cW_1)$ and $K_l(\vx_2) = K_l(\vx_2; \cW_2)$.

In SimCLR case, we have $\cW_1 = \cW_2 = \cW$ so 
\begin{equation}
    \vec\left(\frac{\partial r}{\partial W_l}\right) = K_l(\vx_1)\left[K_l(\vx_1) - K_l(\vx_2)\right]^\t\vec(W_l)
\end{equation}
\end{proof}
\textbf{Remark for $\ell_2$ normalization}. Note that in the presence of $\ell_2$ normalization (Eqn.~\ref{eq:loss-l2-normalization}), with similar derivation, we get: 
\begin{equation}
    \vec\left(\frac{\partial r_\n}{\partial W_l}\right) = K^\n_l(\vx_1)P^\perp_{\vf_L(\vx_1)}\left[K^\n_l(\vx_1) - K^\n_l(\vx_2)\right]^\t\vec(W_l)
\end{equation}
where $K^\n_l(\vx) := K_l(\vx) / \|\vf_L(\vx)\|_2$ and $P^\perp_{\vf_L(\vx_1)}$ is a projection matrix that project the gradient to the orthogonal complementary subspace of $\vf_L(\vx_1)$:
\begin{equation}
    P^\perp_{\vv} := I_{n_L\times n_L} - \frac{\vv}{\|\vv\|_2} \frac{\vv^\t}{\|\vv\|_2}
\end{equation}

\section{Analysis of SimCLR using Teacher-Student Setting (Section~\ref{sec:simclr})}
\label{sec:appendix-analysis}
\subsection{Theorem~\ref{thm:contrastive-loss}}
\begin{theorem}[Common Property of Contrastive Losses]
For loss functions $L \in \{\lsimple{}, \ltriplet^\tau{}, \lnce^\tau\}$, we have $\frac{\partial L}{\partial r_+} > 0$, $\frac{\partial L}{\partial r_{k-}} < 0$ for $1 \le k \le H$ and $\frac{\partial L}{\partial r_+} + \sum_{k=1}^H \frac{\partial L}{\partial r_{k-}} = 0$. 
\end{theorem}
\begin{proof}
For $\lsimple$ and $\ltriplet$ the derivation is obvious. For $\lnce{}$, we have:
\begin{eqnarray}
\frac{\partial L}{\partial r_+} &=& \frac{1}{\tau}\left(1 - \frac{e^{-r_+/\tau}}{e^{-r_+/\tau} + \sum_{k'=1}^H e^{-r_{k'-}/\tau}}\right) > 0 \\
\frac{\partial L}{\partial r_{k-}} &=& - \frac{1}{\tau} \left(\frac{e^{-r_{k-}/\tau}}{e^{-r_+/\tau} + \sum_{k'=1}^H e^{-r_{k'-}/\tau}}\right) < 0, \quad\quad k = 1,\ldots, H
\end{eqnarray}
and obviously we have:
\begin{equation}
    \frac{\partial L}{\partial r_+} + \sum_{k=1}^H \frac{\partial L}{\partial r_{k-}} = 0 \label{eq:balanced-pn}
\end{equation}
\end{proof}

\def\cX{\mathcal{X}}

\subsection{The Covariance Operator under different loss functions}
\begin{lemma}
\label{lemma:cov-grad}
For a loss function $L$ that satisfies Theorem~\ref{thm:contrastive-loss}, with a batch of size one with samples $\cX := \{\vx_1$, $\vx_+, \vx_{1-}, \vx_{2-}, \ldots, \vx_{H-}\}$, where $\vx_1, \vx_+ \sim p_\aug(\cdot|\vx)$ are augmentation from the same sample $\vx$, and $\vx_{k-} \sim p_\aug(\cdot|\vx'_k)$ are augmentations from independent samples $\vx'_k \sim p(\cdot)$. We have:
\begin{equation}
    \vec(g_{W_l}) = K_l(\vx_1)\sum_{k=1}^H \left[\frac{\partial L}{\partial r_{k-}}\Bigg|_\cX \cdot (K_l(\vx_+) - K_l(\vx_{k-}))^\t\right]\vec(W_l)
\end{equation}
\end{lemma}
\begin{proof}
First we have:
\begin{equation}
    \vec(g_{W_l}) = \frac{\partial L}{\partial W_l} = \frac{\partial L}{\partial r_+}\frac{\partial r_+}{\partial W_l} + \sum_{k=1}^H\frac{\partial L}{\partial r_{k-}}\frac{\partial r_{k-}}{\partial W_l} \label{eq:contrast-weight}
\end{equation}

Then we compute each terms. Using Theorem~\ref{thm:l2-two-tower}, we know that:
\begin{eqnarray}
    \frac{\partial r_+}{\partial W_l} &=& K_l(\vx_1)(K_l(\vx_1) - K_l(\vx_+))^\t \vec(W_l)\\
    \frac{\partial r_{k-}}{\partial W_l} &=& K_l(\vx_1)(K_l(\vx_1) - K_l(\vx_{k-}))^\t \vec(W_l), \quad k = 1,\ldots,n
\end{eqnarray}
Since Eqn.~\ref{eq:balanced-pn} holds, $K_l(\vx_1)K_l^\t(\vx_1)$ will be cancelled out and we have:
\begin{equation}
    \vec(g_{W_l}) = K_l(\vx_1)\sum_{k=1}^H \left[\frac{\partial L}{\partial r_{k-}} (K_l(\vx_+) - K_l(\vx_{k-}))^\t\right]\vec(W_l)
\end{equation}
\end{proof}
\textbf{Remark}. In $\ell_2$-normalized loss function, similarly we have:
\begin{equation}
    \vec(g_{W_l}) = K^\n_l(\vx_1)P^\perp_{\vf(\vx_1)} \sum_{k=1}^H \left[\frac{\partial L}{\partial r_{k-}} (K^\n_l(\vx_+) - K^\n_l(\vx_{k-}))^\t\right]\vec(W_l)
\end{equation}

\subsection{Theorem~\ref{thm:contrast-simclr-simple}}
\begin{theorem}[Covariance Operator for $\lsimple$]
With large batch limit, $W_l$'s update under $\lsimple{}$ is $W_l(t+1) = W_l(t) + \alpha \Delta W_l(t)$ ($\alpha$ is the learning rate), where 
\begin{equation}
\vec(\Delta W_l(t)) = \covop_l^\simple(\cW) \vec(W_l(t)). \label{eq:covariance-operator}
\end{equation}
Here $\covop_l^\simple(\cW)\!\!:=\!\!\var_{\vx\sim p(\cdot)}[\bar K_l(\vx;\cW)] \in \rr^{n_ln_{l-1}\times n_ln_{l-1}}$ is the \emph{covariance operator} for $\lsimple{}$, $\bar K_l(\vx;\cW) := \eee{\vx' \sim p_\aug(\cdot|\vx)}{K_l(\vx';\cW)}$ is the expected connection under the augmentation distribution, conditioned on datapoint $\vx$.
\end{theorem}
\begin{proof}
For $\lsimple := r_+ - r_-$, we have $H = 1$ and $\frac{\partial L}{\partial r_{-}} \equiv -1$. Therefore using Lemma~\ref{lemma:cov-grad}, we have: 
\begin{equation}
    \vec(g_{W_l}) = -K_l(\vx_1)\left[ K_l^\t(\vx_+) -  K^\t_l(\vx_{-})\right]\vec(W_l)
\end{equation}
Taking large batch limits, we know that $\ee{K_l(\vx_1)K^\t_l(\vx_+)} = \eee{\vx}{\bar K_l(\vx) \bar K^\t_l(\vx)}$ since $\vx_1, \vx_+ \sim p_\aug(\cdot|\vx)$ are all augmented data points from a common sample $\vx$. On the other hand, $\ee{K_l(\vx_1)K^\t_l(\vx_{k-})} = \eee{\vx}{\bar K_l(\vx)}\eee{\vx}{\bar K^\t_l(\vx)}$ since $\vx_1\sim p_\aug(\cdot|\vx)$ and $\vx_{k-}\sim p_\aug(\cdot|\vx'_k)$ are generated from independent samples $\vx$ and $\vx'_k$ and independent data augmentation. Therefore, 
\begin{eqnarray}
    \vec(g_{W_l}) &=& -\Big\{\eee{\vx}{\bar K_l(\vx) \bar K^\t_l(\vx)} - \eee{\vx}{\bar K_l(\vx)}\eee{\vx}{\bar K^\t_l(\vx)}\Big\}\vec(W_l) \\
    &=& -\var_{\vx}\left[\bar K_l(\vx)\right]\vec(W_l)
\end{eqnarray}
The conclusion follows since gradient descent is used and $\Delta W_l = - g_{W_l}$.
\end{proof}
\textbf{Remark}. In $\ell_2$-normalized loss function, similarly we have:
\begin{equation}
    \vec(g_{W_l}) = -\Cov_{\vx}\left[\overline{K^{\n}_l(\vx)P^\perp_{\vf(\vx)}}, \bar K^\n_l(\vx)\right]\vec(W_l)
\end{equation}
Note that it is no longer symmetric. We leave detailed discussion to the future work. 

\def\noaug{\mathrm{noaug}}

\subsection{Theorem~\ref{thm:contrast-simclr-pairwise}}
\begin{theorem}[Covariance Operator for $\ltriplet^\tau{}$ and $\lnce^\tau$ ($H = 1$, single negative pair)]
Let $r := \frac{1}{2}\|\vf_L(\vx) - \vf_L(\vx')\|_2^2$. The covariance operator $\covop_l(\cW)$ has the following form:
\begin{equation}
\covop_l(\cW) = \frac{1}{2}\var^\xi_{\vx,\vx'\sim p(\cdot)}\left[\bar K_l(\vx) - \bar K_l(\vx')\right] + \theta \label{eq:weighted-cov-op}
\end{equation}
where the pairwise weight $\xi(r)$ takes the following form for different loss functions:
\begin{equation}
   \xi(r) = \left\{
   \begin{array}{cc}
   1 & L = \lsimple \\
    \frac{e^{-(r-r_0)/\tau}}{1 + e^{-(r-r_0)/\tau}} & L = \ltriplet^\tau \\
    \frac{1}{\tau} \frac{e^{-r/\tau}}{1 + e^{-r/\tau}} & L = \lnce^\tau{}
   \end{array}
   \right. \label{eq:xi}
\end{equation}
and $\theta\!:=\!\mathcal{O}(\eee{\vx,\vx'}{\sqrt{r(\vx,\vx')}\sigma_\aug(\vx)}\!+\!\eee{\vx}{\sigma^2_\aug(\vx)})$ is the residue term. $\sigma^2_\aug(\vx) := \tr\var_{\vx''\sim p_\aug(\cdot|\vx)}[\vf_L(\vx'')]$ and $\mathrm{tr}$ is the trace of a matrix. For $\lsimple$, $\theta\equiv 0$.  
\end{theorem}
\begin{proof}
When $\partial L / \partial r_{k-}$ is no longer constant, we consider its expansion with respect to un-augmented data point $\cX_0 = \{\vx, \vx'_1, \ldots, \vx'_k\}$. Here $\cX = \{\vx_1, \vx_+, \vx_{1-}, \ldots, \vx_{H-}\}$ is one data sample that includes both positive and negative pairs. Note that $\vx_1, \vx_+ \sim p_\aug(\cdot|\vx)$ and $\vx_{k-}\sim p_\aug(\cdot|\vx'_k)$ for $1\le k\le H$.
\begin{equation}
\frac{\partial L}{\partial r_{k-}}\Bigg|_\cX = \frac{\partial L}{\partial r_{k-}}\Bigg|_{\cX_0} + \epsilon 
\end{equation}
where $\epsilon$ is a bounded quantity for $\ltriplet{}$ ($|\epsilon| \le 1$) and $\lnce{}$ ($|\epsilon| \le 2 / \tau$). 

We consider $H = 1$ where there is only a single negative pair (and $r_-$). In this case $\cX_0 = \{\vx,\vx'\}$. Let 
\begin{equation}
r:= \frac12\|\vf_L(\vx) - \vf_L(\vx')\|^2_2
\end{equation}
and 
\begin{equation}
\xi(\vx,\vx') := - \frac{\partial L}{\partial r_{-}}\big|_{\cX_0}
\end{equation}

Note that for $\ltriplet{}$, it is not differentiable, so we could use its soft version: $\ltriplet^\tau(r_+, r_-) = \tau \log(1 + e^{(r_+ - r_- + r_0) / \tau})$. It is easy to see that $\lim_{\tau \rightarrow 0} \ltriplet^\tau(r_+, r_-) \rightarrow \max(r_+ - r_- + r_0, 0)$. 

For the two losses:
\begin{itemize}
    \item For $\ltriplet^\tau$, we have 
    \begin{equation}
        \xi(\vx,\vx') = \xi(r) = \frac{e^{-r/\tau}}{e^{-r_0/\tau} + e^{-r/\tau}}. \label{eq:xi-triplet}
    \end{equation}
    \item For $\lnce^\tau$, we have 
    \begin{equation}
        \xi(\vx,\vx') = \xi(r) = \frac{1}{\tau}\frac{e^{-r/\tau}}{1 + e^{-r/\tau}} \label{eq:xi-nce}.
    \end{equation} 
\end{itemize}
Note that for $\ltriplet^\tau$ we have $\lim_{\tau\rightarrow 0} \xi(r) = \mathbb{I}(r \le r_0)$.  for $\lnce{}$, since it is differentiable, by Taylor expansion we have $\epsilon =  
O(\|\vx_1-\vx\|_2,\|\vx_+-\vx\|_2, \|\vx_- - \vx'\|_2)$, which will be used later. 

\textbf{The constant term $\xi$ with respect to data augmentation.} In the following, we first consider the term $\xi$, which only depends on un-augmented data points $\cX_0$. From Lemma~\ref{lemma:cov-grad}, we now have a term in the gradient:
\begin{equation}
    \vg_l(\cX) := -K_l(\vx_1)\left[ K_l^\t(\vx_+) -  K^\t_l(\vx_{-})\right]\xi(\vx, \vx')\vec(W_l)
\end{equation}
Under the large batch limit, taking expectation with respect to data augmentation $p_\aug$ and notice that all augmentations are done independently, given un-augmented data $\vx$ and $\vx'$, we have:
\begin{equation}
    \vg_l(\vx, \vx') := \eee{p_\aug}{\vg_l(\cX)} = -\bar K_l(\vx)\left[ \bar K_l^\t(\vx) - \bar K^\t_l(\vx')\right]\xi(\vx, \vx')\vec(W_l)
\end{equation}
Symmetrically, if we swap $\vx$ and $\vx'$ since both are sampled from the same distribution $p(\cdot)$, we have:
\begin{equation}
    \vg_l(\vx', \vx) =  -\bar K_l(\vx')\left[ \bar K_l^\t(\vx') - \bar K^\t_l(\vx)\right]\xi(\vx', \vx)\vec(W_l)
\end{equation}
since $\xi(\vx', \vx)$ only depends on the squared $\ell_2$ distance $r$ (Eqn.~\ref{eq:xi-triplet} and Eqn.~\ref{eq:xi-nce}), we have $\xi(\vx', \vx) = \xi(\vx,\vx') = \xi(r)$ and thus:
\begin{eqnarray}
    \vg_l(\vx, \vx') + \vg_l(\vx', \vx) &=& -\left[\bar K_l(\vx)\bar K_l^\t(\vx) - \bar K_l(\vx)\bar K^\t_l(\vx') + \bar K_l(\vx')\bar K_l^\t(\vx') - \bar K_l(\vx')\bar K^\t_l(\vx)\right]\xi(r)\vec(W_l) \nonumber \\
    &=& - \xi(r) (\bar K_l(\vx) - \bar K_l(\vx'))(\bar K_l(\vx) - \bar K_l(\vx'))^\t\vec(W_l) \label{eq:symmetry-trick}
\end{eqnarray}
Therefore, we have:
\begin{eqnarray}
    \eee{\vx,\vx'\sim p}{\vg_l(\vx,\vx')} &=& -\frac{1}{2}\eee{\vx,\vx'\sim p}{\xi(r) (\bar K_l(\vx) - \bar K_l(\vx'))(\bar K_l(\vx) - \bar K_l(\vx'))^\t}\vec(W_l) \\
    &=& -\frac{1}{2}\var^\xi_{\vx,\vx'\sim p}\left[\bar K_l(\vx) - \bar K_l(\vx')\right]\vec(W_l)
\end{eqnarray}

\begin{figure}
    \centering
    \includegraphics[width=0.8\textwidth]{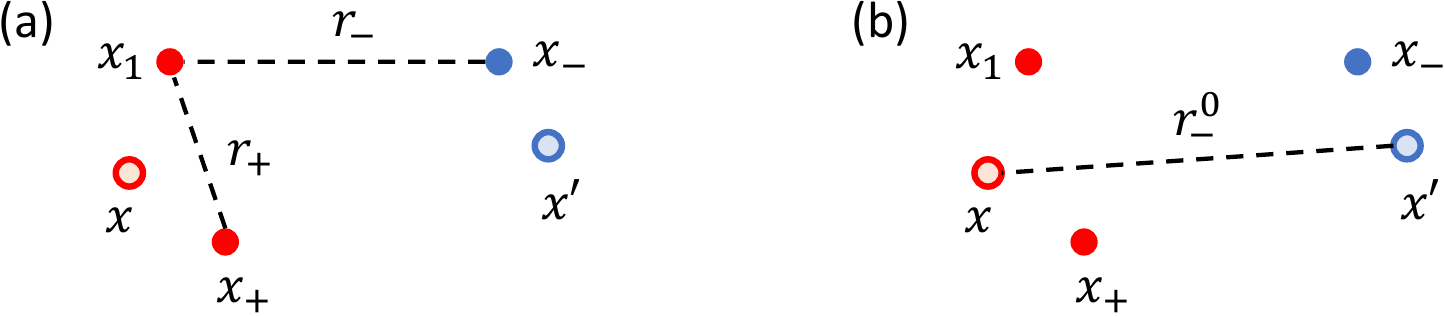}
    \caption{\small Notations used in Theorem~\ref{thm:contrast-simclr-pairwise}. \textbf{(a)} Un-augmented data point $\vx, \vx'\sim p(\cdot)$ and the augmented data points $\vx_1, \vx_+ \sim p_\aug(\cdot|\vx)$ (in red) and $\vx_- \sim p_\aug(\cdot|\vx')$ (in blue). The squared distance $r_+ := \frac12\|\vf_L(\vx_1) - \vf_L(\vx_+)\|_2^2$ and $r_- := \frac12\|\vf_L(\vx_1) - \vf_L(\vx_-)\|_2^2$. \textbf{(b)} $r_-^0 := \frac12\|\vf_L(\vx) - \vf_L(\vx')\|_2^2$ used in bounding the residue term $\theta$.}
    \label{fig:proof-theorem4}
\end{figure}

\textbf{Bound the error}. For $\lnce{}$, let 
\begin{equation}
F := -\frac{\partial L}{\partial r_-} = \frac{1}{\tau} \left(\frac{e^{-r_-/\tau}}{e^{-r_+/\tau} +  e^{-r_-/\tau}}\right)  > 0
\end{equation}
then it is clear that $0 < F < 1 / \tau$. We can compute its partial derivatives:
\begin{equation}
\frac{\partial F}{\partial r_+}= -F(1 / \tau - F),\quad\quad \frac{\partial F}{\partial r_-}= F(1 / \tau - F) 
\end{equation}
Note that $|F(1/ \tau - F)| < 1 / \tau^2$ is always bounded. 

Note that $F$ is a differentiable function with respect to $r_+$ and $r_-$. We do a Taylor expansion of $F$ on the two variables $(r_+,r_-)$, using intermediate value theorem, we have:
\begin{equation}
    \epsilon := -F\Big|_{\cX} + F\Big|_{\cX_0} = -\frac{\partial F}{\partial r_+}\Big|_{\{\tilde r_+, \tilde r_-\}} (r_+ - r_+^0) - \frac{\partial F}{\partial r_-} \Big|_{\{\tilde r_+, \tilde r_-\}} (r_- - r^0_-) 
\end{equation}
for derivatives evaluated at some point $\{\tilde r_+, \tilde r_-\}$ at the line connecting $(\vx, \vx, \vx')$ and $(\vx_1, \vx_+, \vx_-)$. $r_+^0$ and $r_-^0$ are squared $\ell^2$ distances  evaluated at $(\vx, \vx, \vx')$, therefore, $r_+^0 \equiv 0$ and $r^0_- = \frac12 \|\vf(\vx) - \vf(\vx')\|_2^2$ (note that here we just use $\vf := \vf_L$ for brevity). 

Therefore, we have $r_+ - r^0_+ = \frac12 \|\vf(\vx_1) - \vf(\vx_+)\|_2^2$ and we have:
\begin{eqnarray}
    \eee{p_\aug}{\left|\frac{\partial F}{\partial r_+}\Big|_{\{\tilde r_+, \tilde r_-\}} (r_+ - r_+^0)\right|} &\le& \frac{1}{\tau^2} \cdot \frac12 \int \|\vf(\vx_1) - \vf(\vx_+)\|_2^2 p_\aug(\vx_1|\vx) p_\aug(\vx_+|\vx)\dd\vx_1\dd\vx_+ \nonumber \\
    &=& \frac{1}{\tau^2} \left\{\eee{\vx' \sim p_\aug(\cdot|\vx)}{\|\vf(\vx')\|^2} - \|\eee{\vx' \sim p_\aug(\cdot|\vx)}{\vf(\vx')}\|^2\right\} \\ 
    &=& \frac{1}{\tau^2} \tr \var_\aug[\vf|\vx]
\end{eqnarray}
where $\tr\var_\aug[\vf|\vx] := \tr\var_{\vx' \sim p_\aug(\cdot|\vx)}[\vf(\vx')]$ is a scalar. 

Similarly, for $r_- - r^0_-$ we want to break the term into groups of terms, each is a difference \emph{within} data augmentation. Using that 
\begin{equation}
\|\va + \Delta\va - (\vb + \Delta\vb)\|_2^2 - \|\va-\vb\|_2^2 = 2(\Delta\va - \Delta\vb)^\t(\va-\vb) + \|\Delta\va-\Delta\vb\|_2^2
\end{equation}
we have (here $\va := \vf(\vx)$, $\vb := \vf(\vx')$, $\Delta\va := \vf(\vx_1) - \vf(\vx)$ and $\Delta\vb := \vf(\vx_-) - \vf(\vx')$):
\begin{eqnarray}
    r_- - r^0_- &=& \frac12\left[\|\vf(\vx_1) - \vf(\vx_-)\|_2^2 - \|\vf(\vx) - \vf(\vx')\|_2^2\right] \\
    &=& \left[\vf(\vx) - \vf(\vx')\right]^\t \left[(\vf(\vx_1) - \vf(\vx)) - (\vf(\vx_-) - \vf(\vx'))\right] \nonumber \\
    &+& \frac12\|(\vf(\vx_1) - \vf(\vx)) - (\vf(\vx_-) - \vf(\vx'))\|^2 
\end{eqnarray}
Using $\frac{1}{2}\|\va-\vb\|_2^2 \le  \|\va\|_2^2 + \|\vb\|_2^2$ and Eqn.~\ref{eq:e-abs-value}, we have the following (Here $c_0(\vx) := \|\vf(\vx) - \eee{\vx'\sim p_\aug(\cdot|\vx)}{\vf(\vx')}\|$:
\begin{eqnarray}
    & & \eee{p_\aug}{\left|\frac{\partial F}{\partial r_-}\Big|_{\{\tilde r_+, \tilde r_-\}} (r_- - r_-^0)\right|} \\ &\le& \frac{1}{\tau^2} \left\{ \|\vf(\vx)-\vf(\vx')\|\left(\sqrt{\tr\var_\aug[\vf|\vx]} + \sqrt{\tr\var_\aug[\vf|\vx']} + c_0(\vx) + c_0(\vx')\right) + \tr\var_\aug[\vf|\vx] + \tr\var_\aug[\vf|\vx'] \right\} \nonumber
\end{eqnarray}

Let $M_K := \max_\vx \|K_l(\vx)\|$ so finally we have: 
\begin{eqnarray}
\theta &=& |\eee{\vx,\vx',\aug}{\epsilon K_l(\vx_1)(K^\t_l(\vx_+) - K^\t_l(\vx_-))}| \nonumber \\
&\le& \eee{\vx,\vx',\aug}{\left|\epsilon K_l(\vx_1)(K^\t_l(\vx_+) - K^\t_l(\vx_-))\right|} \nonumber \\ 
&\le &\frac{2M^2_K}{\tau^2} \left\{2\eee{\vx,\vx'\sim p(\cdot)}{\|\vf(\vx)-\vf(\vx')\|\left(\sqrt{\tr\var_\aug[\vf|\vx]}+c_0(\vx)\right)} + 3\tr\eee{\vx}{\var_\aug[\vf|\vx]}\right\} \label{eq:final-nce-bound} 
\end{eqnarray}
Note that if there is no augmentation (i.e., $p_\aug(\vx_1|\vx) = \delta(\vx_1 - \vx)$), then $c_0(\vx) \equiv 0$, $\var_\aug[\vf|\vx] \equiv 0$ and the error (Eqn.~\ref{eq:final-nce-bound}) is also zero. A small range of augmentation yields tight bound.  

For $\ltriplet^\tau$, the derivation is similar. The only difference is that we have $1 / \tau$ rather than $1 / \tau^2$ in Eqn.~\ref{eq:final-nce-bound}. Note that this didn't change the order of the bound since $\xi(r)$ (and thus the covariance operator) has one less $1 / \tau$ as well. We could also see that for hard loss $\ltriplet$, since $\tau \rightarrow 0$ this bound will be very loose. We leave a more tight bound as the future work.  
\end{proof}

\textbf{Remarks for $H > 1$.} Note that for $H > 1$, $\lnce{}$ has multiple negative pairs and $\partial L / \partial r_{k-} = e^{-r_{k-} / \tau} / Z(\cX)$ where $Z(\cX) := e^{-r_+/\tau} + \sum_{k=1}^H e^{-r_{k-}/\tau}$. While the nominator $e^{-r_{k-} / \tau}$ still only depends on the distance between $\vx_1$ and $\vx_{k-}$ (which is good), the normalization constant $Z(\cX)$ depends on $H + 1$ distance pairs simultaneously. This leads to 

\begin{equation}
    \xi_k = \frac{\partial L}{\partial r_{k-}}\Big|_{\cX_0} = \frac{e^{-\|\vx - \vx'_k\|^2_2 / \tau}}{1 + \sum_{k=1}^H e^{-\|\vx - \vx'_k\|_2^2 / \tau}}
\end{equation}
which causes issues with the symmetry trick (Eqn.~\ref{eq:symmetry-trick}), because the denominator involves many negative pairs at the same time.  

However, if we think given one pair of distinct data point $(\vx, \vx')$, the normalized constant $Z$ averaged over data augmentation is approximately constant due to homogeneity of the dataset and data augmentation, then Eqn.~\ref{eq:symmetry-trick} can still be applied and similar conclusion follows. 
\begin{corollary}
SimCLR under $\lsimple^\beta := (1 + \beta) r_+ - r_-$ has the gradient update rule as follows at layer $l$:
\begin{equation}
    \vec(\Delta W_l) = \op_l\vec(W_l) = (-\beta \opintra_l + \opinter_l)\vec(W_l)
\end{equation}
where $\opintra_l$ and $\opinter_l$ are intra-augmentation and inter-augmentation covariance operators at layer $l$:
\begin{eqnarray}
    \opintra_l &:=& \eee{\vx\sim p(\cdot)}{\var_{\vx'\sim p_\aug(\cdot|\vx)}[K_l(\vx')]} \\
    \opinter_l &:=& \var_{\vx\sim p(\cdot)}\left[\bar K_l(\vx)\right] = \var_{\vx\sim p(\cdot)}\left[\eee{\vx'\sim p_\aug(\cdot|\vx)}{K_l(\vx')}\right]
\end{eqnarray}
\end{corollary}
\begin{proof}
Note that the second part $\var_{\vx\sim p(\cdot)}\left[\bar K_l(\vx)\right]$ (which corresponds to $r_+ - r_-$) has already proven in Theorem 3 of the main paper. By gradient descent, we have:
\begin{equation}
    \vec(g_{W_l}) = \frac{\partial L}{\partial W_l} = \frac{\partial L}{\partial r_+}\frac{\partial r_+}{\partial W_l} + \frac{\partial L}{\partial r_{-}}\frac{\partial r_{-}}{\partial W_l} \label{eq:contrast-weight}
\end{equation}
and by Lemma~\ref{lemma:cov-grad}:
\begin{equation}
    \frac{\partial r_+}{\partial W_l} = K_l(\vx_1)(K_l(\vx_1) - K_l(\vx_+))^\t \vec(W_l)
\end{equation}
If we take expectation conditioned on $\vx$, the common sample that generates both $\vx_1$ and $\vx_+$ with $p_\aug(\cdot|\vx)$, the first term becomes $\eee{\vx'\sim p_\aug}{K_l K_l^\t}$ and the second term becomes $\eee{\vx'\sim p_\aug}{K_l}\eee{\vx'\sim p_\aug}{K^\t_l}$. so we have:
\begin{equation}
    \eee{\vx_1,\vx_+}{\frac{\partial r_+}{\partial W_l}} = \eee{\vx\sim p(\cdot)}{\var_{\vx'\sim p_\aug(\cdot|\vx)}[K_l(\vx')]}\vec(W_l)
\end{equation}
The conclusion follows since gradient descent is used and $\Delta W_l = - g_{W_l}$.
\end{proof}

\def\diag{\mathrm{diag}}

\section{Hierarchical Latent Tree Models (Section~\ref{sec:motivatehltm})}
\subsection{Lemmas}
\begin{lemma}[Variance Squashing]
\label{lemma:var-squashing}
Suppose a function $\phi: \rr\mapsto \rr$ is $L$-Lipschitz continuous: $|\phi(x) - \phi(y)| \le L |x - y|$, then for $x \sim p(\cdot)$, we have:
\begin{equation}
    \var_p[\phi(x)] \le L^2 \var_p[x]
\end{equation}
\end{lemma}
\begin{proof}
Suppose $x, y\sim p(\cdot)$ are independent samples and $\mu_\phi := \ee{\phi(x)}$. Note that $\var[\phi(x)]$ can be written as the following:
\begin{eqnarray}
    \ee{|\phi(x)-\phi(y)|^2} &=& \frac{1}{2}\ee{|(\phi(x)- \mu_\phi) - (\phi(y) - \mu_\phi)|^2} \nonumber \\
    &=& \ee{|\phi(x)- \mu_\phi|^2} + \ee{|\phi(y)- \mu_\phi|^2} - 2\ee{(\phi(x)- \mu_\phi)(\phi(y)- \mu_\phi)} \nonumber \\
    &=& 2\var_p[\phi(x)] 
\end{eqnarray}
Therefore we have:
\begin{equation}
   \var_p[\phi(x)] = \frac{1}{2}\ee{|\phi(x)-\phi(y)|^2} \le \frac{L^2}{2}\ee{|x-y|^2} = L^2\var_p[x] \label{eq:var-squashing}
\end{equation}
\end{proof}

\begin{lemma}[Sharpened Jensen's inequality~\cite{liao2018sharpening}]
If function $\phi$ is twice differentiable, and $x\sim p(\cdot)$, then we have:
\begin{equation}
    \frac{1}{2}\var[x]\inf\phi'' \le \ee{\phi(x)} - \phi(\ee{x}) \le \frac{1}{2}\var[x]\sup\phi'' 
\end{equation}
\end{lemma}

\begin{lemma}[Sharpened Jensen's inequality for ReLU activation]
\label{lemma:sharpened-jensen}
For ReLU activation $\psi(x) := \max(x, 0)$ and $x\sim p(\cdot)$, we have:
\begin{equation}
0 \le \ee{\psi(x)} - \psi(\ee{x}) \le \sqrt{\var_p[x]} 
\end{equation}
\end{lemma}
\begin{proof}
Since $\psi$ is a convex function, by Jensen's inequality we have $\ee{\psi(x)} - \psi(\ee{x}) \ge 0$. For the other side, let $\mu := \eee{p}{x}$ and we have (note that for ReLU, $\psi(x) - \psi(\mu) \le |x - \mu|$):
\begin{eqnarray}
  \ee{\psi(x)} - \psi(\ee{x}) &=& \int (\psi(x) - \psi(\mu)) p(x)\dd x \\
  &\le& \int |x - \mu|p(x)\dd x
\end{eqnarray}
Note that for the expectation of absolute value $|x-\mu|$, we have:
\begin{equation}
  \int |x - \mu|p(x)\dd x \le \left(\int |x - \mu|^2 p(x)\dd x\right)^{1/2} \left(\int p(x)\dd x\right)^{1/2} = \sqrt{\var_p[x]} \label{eq:e-abs-value}
\end{equation}
where the last inequality is due to Cauchy-Schwarz. 
\end{proof}

\begin{table}[]
    \centering
    \small
    \setlength{\tabcolsep}{1.8pt}
    \begin{tabular}{l|l|l|l}
        \textbf{Symbol} & \textbf{Definition} & \textbf{Size} & \textbf{Description}  \\
        \hline
        $\latents{l}$ & & & The set of all latent variables at layer $l$ of the generative model. \\
        $\nodes{l}$ & & & The set of all neurons at layer $l$ of the neural network. \\
        $\nodes{\mu}$ & &  & The set of neurons that corresponds to $z_\mu$. \\
        $\ch{\mu}$ & $\bigcup_{\nu \in \mathrm{ch}(\mu)} \nodes{\nu}$ &  & The set of neurons that corresponds to children of latent $z_\mu$. \\
        $\cardinal\mu$ & & & Number of possible categorical values taken by $z_\mu \in \{0,\dots,m_\mu-1\}$. \\
        \hline
        $\vzero_\mu$, $\vone_\mu$ & & $\cardinal\mu$ & All-one and all-zero vectors. \\
        $P_{\mu\nu}$ & $[\pr(z_\nu|z_\mu)]$ & $\cardinal\mu\times \cardinal\nu$ & The top-down transition probability from $z_\mu$ to $z_\nu$. \\ 
        $\rho_{\mu\nu}$ & $2 \pr(z_\nu\!\!=\!\!1|z_\mu\!\!=\!\!1) - 1$ & scalar in $[-1,1]$ & Polarity of the transitional probability in the binary case. \\
        $P_0$ & $\diag[\pr(z_0)]$ & $\cardinal0\times \cardinal0$ & The diagonal matrix of probability of $z_0$ taking different values. \\ 
        $v_j(z_\mu)$ & $\eee{z}{f_j|z_\mu}$ & scalar & Expectation of activation $f_j$ given $z_\mu$ ($z_\mu$'s descendants are marginalized). \\
        $\vv_j$ & $[v_j(z_\mu)]$ & $\cardinal\mu$ & Vector form of $v_j(z_\mu)$. \\
        $\vf_\mu$, $\vf_{\ch{\mu}}$ & $[f_j]_{j\in\nodes{\mu}}$, $[f_k]_{k\in\ch{\mu}}$ & $|\nodes{\mu}|$, $|\ch{\mu}|$ & Activations for all nodes $j \in \nodes{\mu}$ and for the children of $\nodes{\mu}$ \\ 
        $\vv_{0k}$, $V_{0,\ch{\mu}}$ & $[\eee{z}{f_k|z_0}]$, $[\vv_{0k}]_{k\in\ch{\mu}}$ & $\cardinal0$, $\cardinal0\times |\ch{\mu}|$ & Expected activation conditioned on $z_0$ \\
        $s_k$ & $\frac{1}{2}(v_k(1) - v_k(0))$ & scalar & Discrepancy of node $k$ w.r.t its latent variable $z_{\nu(k)}$. \\
        $\vamu$ & $[\rho_{\mu\nu(k)}s_k]_{k\in\ch{\mu}}$ & $|\ch{\mu}|$ & Child selectivity vector in the binary case. 
    \end{tabular}
    \caption{Extended notation in \hltm{}.}
    \label{tab:notation-tree-extended}
\end{table}

\begin{figure}
    \centering
    \includegraphics[width=\textwidth]{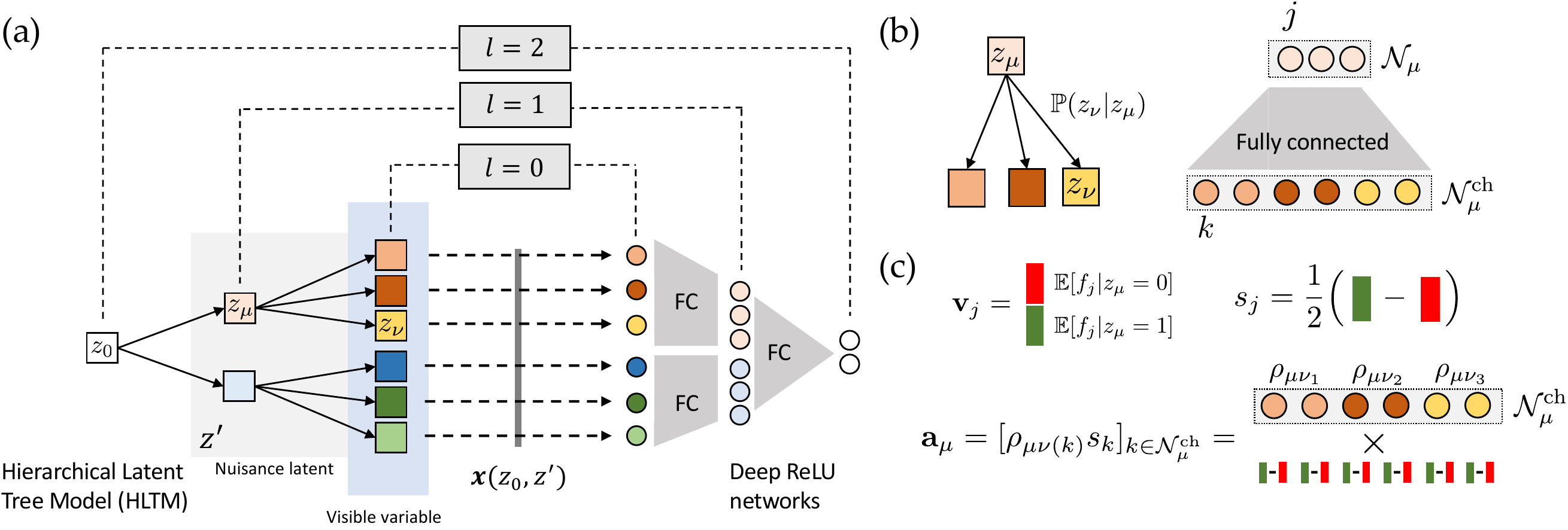}
    \caption{\textbf{(a)} The Hierarchical Latent Tree Model (HLTM). \textbf{(b)} Correspondence between latent variables and nodes in the intermediate layer of neural networks. \textbf{(c)} Definition of $\vv_j$, $s_j$ and $\vamu$ in Table.~\ref{tab:notation-tree-extended}.}
    \label{fig:hltm-appendix}
\end{figure}

\subsection{Taxonomy of \hltm{}}
For convenience, we define the following symbols for $k\in\ch{\mu}$ (note that $|\ch{\mu}|$ is the number of the children of the node set $\nodes{\mu}$): 
\begin{eqnarray}
    \vv_{\mu k} &:=&\eee{z}{f_k|z_\mu} = P_{\mu\nu(k)} \vv_k \in\rr^{\cardinal\mu} \\
    V_{\mu,\ch{\mu}} &:=& [\vv_{\mu k}]_{k\in\ch{\mu}}\\ 
    \tilde\vv_j &:=& \left[\eee{z}{\tilde f_j|z_\mu}\right] = V_{\mu,\ch{\mu}}\vw_j \in \rr^{\cardinal\mu}  
\end{eqnarray}

\begin{figure}
    \centering
    \includegraphics[width=0.5\textwidth]{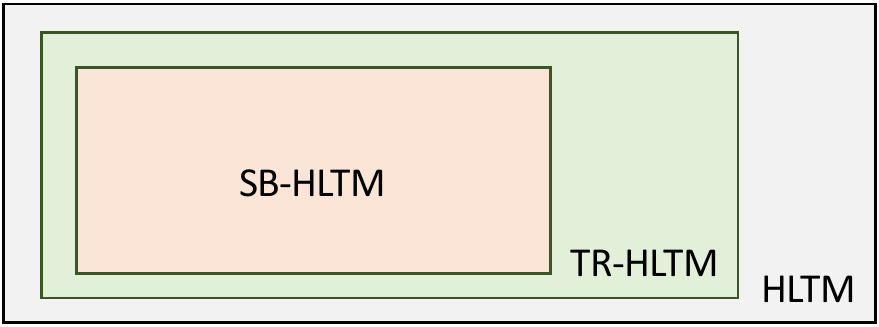}
    \caption{Taxonomy of \hltm{}. TR-\hltm{} is defined in Def.~\ref{definition:trans-hltm} that allows categorical latent variables but requires their transitional probability satisfies certain conditions. SB-\hltm{} (defined in Sec.~\ref{sec:sb-hltm}) is TR-\hltm{} if all latent variables are binary (See Lemma~\ref{lemma:tr-sb-relation}).} 
    \label{fig:taxonomy-of=hltm}
\end{figure}

\begin{definition}[Transition-Regularized \hltm{} (abbr. TR-\hltm{})]
\label{definition:trans-hltm}
For $\mu\in \latents{l}$ and $\nu\in\latents{l-1}$, the transitional probability matrix $P_{\mu\nu} := [\pr(z_\nu|z_\mu)]$ has decomposition $P_{\mu\nu} = \frac{1}{\cardinal\nu}\vone_{\mu}\vone^\t_\nu + C_{\mu\nu}$ where $C_{\mu\nu}\vone_\nu = \vzero_\mu$ and $\vone^\t_\mu C_{\mu\nu} = \vzero_\nu$.  
\end{definition}
Note that $C_{\mu\nu}\vone = \vzero$ is obvious due to the property of conditional probability. The real condition is $\vone^\t_\mu C_{\mu\nu} = \vzero_\nu$. If $\cardinal\mu = \cardinal\nu$, then $P_{\mu\nu}$ is a square matrix and Def.~\ref{definition:trans-hltm} is equivalent to $P_{\mu\nu}$ is double-stochastic. The Def.~\ref{definition:trans-hltm} makes computation of $P_{\mu\nu}$ easy for any $z_\mu$ and $z_\nu$. 

As we will see in the remark of Lemma~\ref{lemma:probability-transition}, symmetric binary \hltm{} (SB-\hltm{}) naturally satisfies  Def.~\ref{definition:trans-hltm} and is a TR-\hltm{}. Conversely, TR-\hltm{} allows categorical latent variables and is a super set of SB-\hltm{}. See Fig.~\ref{fig:taxonomy-of=hltm}.

\begin{lemma}[Property of TR-\hltm{}]
\label{lemma:probability-transition}
For TR-\hltm{} (Def.~\ref{definition:trans-hltm}), for $\mu\in\latents{l}$, $\nu\in\latents{1-1}$ and $\alpha\in\latents{l-2}$, we have: 
\begin{equation}
P_{\mu\alpha} = P_{\mu\nu} P_{\nu\alpha} = \frac{1}{\cardinal\alpha}\vone_{\mu}\vone^\t_\alpha+ C_{\mu\nu}C_{\nu\alpha}
\end{equation}
In general, for any $\mu\in \nodes{l_1}$ and $\alpha\in\nodes{l_2}$ with $l_1 > l_2$, we have:
\begin{equation}
P_{\mu\alpha} = \frac{1}{\cardinal\alpha}\vone_{\mu}\vone^\t_\alpha+ \prod_{\mu,\ldots,\xi,\zeta,\ldots,\alpha} C_{\xi\zeta}
\end{equation}
\end{lemma}
\begin{proof}
By Def.~\ref{definition:trans-hltm}, we have
\begin{eqnarray}
P_{\mu\alpha} &=& P_{\mu\nu} P_{\nu\alpha} \\ 
&=& \left(\frac{1}{\cardinal\nu}\vone_{\mu}\vone^\t_\nu + C_{\mu\nu}\right)\left(\frac{1}{\cardinal\alpha}\vone_{\nu}\vone^\t_\alpha + C_{\nu\alpha} \right)
\end{eqnarray}
since $\vone_\nu^\t\vone_\nu = \cardinal\nu$, $C_{\mu\nu}\vone_\nu = \vzero_\nu$ and $\vone^\t_\nu C_{\nu\alpha} = \vzero_\alpha$, the conclusion follows. 
\end{proof}

\def\sel#1#2{s_{#1#2}}
\def\tsel#1#2{{\tilde s}_{#1#2}}

\begin{lemma}
\label{lemma:tr-sb-relation}
TR-\hltm{} with all latent variables being binary (or binary TR-\hltm{}) is equivalent to SB-\hltm{}.
\end{lemma}
\begin{proof}
\textbf{SB-\hltm{} is TR-\hltm{}}. For symmetric binary \hltm{} (SB-\hltm{}) defined in Sec.~\ref{sec:sb-hltm}, we could define $C_{\mu\nu}$ as the following (here $\vq := [-1,1]^\t$): 
\begin{equation}
    C_{\mu\nu} = C_{\mu\nu}(\rho_{\mu\nu}) =  \frac{1}{2}\left[
    \begin{array}{cc}
        \rho_{\mu\nu} & - \rho_{\mu\nu} \\ 
        - \rho_{\mu\nu} & \rho_{\mu\nu} 
    \end{array}
    \right] = \frac{1}{2}\rho_{\mu\nu} \vq\vq^\t 
\end{equation}
And it is easy to verify this satifies the definition of TR-\hltm{} (Def.~\ref{definition:trans-hltm}). 

\textbf{Binary TR-\hltm{} is SB-\hltm{}}. To see this, note that $\vone_2^\t C_{\mu\nu} = \vzero_2$ and $C_{\mu\nu} \vone_2 = \vzero_2$ provides 4 linear constraints (1 redundant), leaving 1 free parameter, which is the polarity $\rho_{\mu\nu}$. It is easy to verify that $\rho_{\mu\nu} \in [-1,1]$ in order to keep $0 \le P_{\mu\nu} \le 1$ a probabilistic measure. Moreover, since $\vq^\t\vq = 2$, the parameterization is close under multiplication:
\begin{equation}
    C(\rho_{\mu\nu})C(\rho_{\nu\alpha}) = \frac{1}{4}\vq\vq^\t\vq\vq^\t \rho_{\mu\nu}\rho_{\nu\alpha} = \frac{1}{2}\vq\vq^\t\rho_{\mu\nu}\rho_{\nu\alpha} = C(\rho_{\mu\nu}\rho_{\nu\alpha}) \label{eq:close-multiplication}
\end{equation}
\end{proof}

\def\vzero{\mathbf{0}}

\subsection{SB-\hltm{} (Sec.~\ref{sec:sb-hltm})}
Since we are dealing with binary case, we define the following for convenience:
\begin{eqnarray}
    v^+_k &:=& v_k(1) \\
    v^-_k &:=& v_k(0) \\
    \bar v_k &:=& \frac{1}{2}(v^+_k + v^-_k) =  \frac{1}{2}\left(\eee{z}{f_k | z_\nu = 1} + \eee{z}{f_k | z_\nu = 0}\right)\\
    \bar\vv_{\ch{\mu}} &:=& [\bar v_k]_{k\in \ch{\mu}} \in \rr^{|\ch{\mu}|} \\
    s_k &:=& \frac{1}{2}(v^+_k - v^-_k) = \frac{1}{2}\left(\eee{z}{f_k | z_\nu = 1} - \eee{z}{f_k | z_\nu = 0}\right) \\
    \vs_{\ch{\mu}} &:=& [s_k]_{k\in \ch{\mu}} \in \rr^{|\ch{\mu}|} \\
    \vamu &:=& [\rho_{\mu\nu(k)}s_k]_{k\in \ch{\mu}}
\end{eqnarray}

\begin{lemma}[Lower-bound of tail distribution of 2D Gaussian]
\label{lemma:2d-gaussian-lower-bound}
For standardized Gaussian variable $y_+$ and $y_-$ with (here $\gamma \ge 0$): 
\begin{equation}
    \vy := \left[
    \begin{array}{c}
        y_+ \\
        y_- 
    \end{array} 
    \right] \sim \mathcal{N}\left(\vzero, \Sigma_\vy\right),\quad \mathrm{where} \ \ \Sigma_\vy :=
     \left[
     \begin{array}{cc}
        1 & -\gamma \\
        -\gamma & 1 
     \end{array} 
     \right]
\end{equation}
Then we have lower-bound for the probability:
\begin{equation}
    \pr\left(\vy \ge \vy_0 := \left[
    \begin{array}{c}
        \sqrt{c}\\
        0
    \end{array}
    \right]\right) \ge \frac{1}{2\pi} \exp\left(-\frac{c}{2}\right) (1-\gamma^2)^{1/2} \hat R(c, \gamma)
\end{equation}
where $\hat R(c,\gamma)$ is defined in Eqn.~\ref{eq:residue}. 
\end{lemma}
\begin{proof}
\begin{figure}
    \centering
    \includegraphics[width=0.5\textwidth]{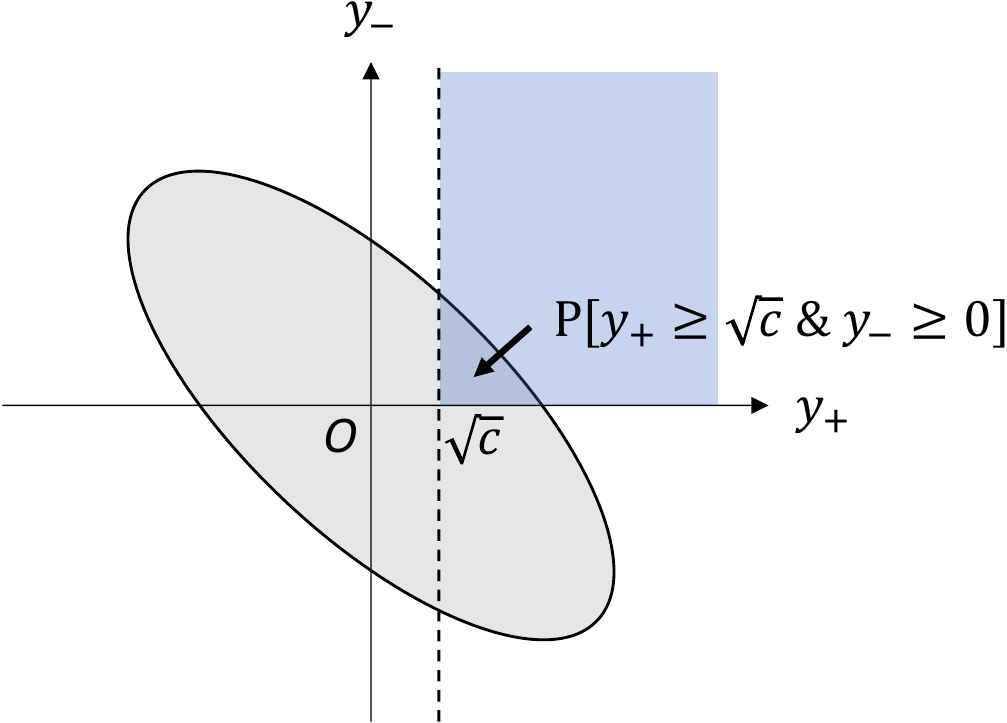}
    \caption{Joint distribution of $\vy = [y_+, y_-]^\t$ and the region we want to lower-bound.}
    \label{fig:lucky-node-2d-gaussian}
\end{figure}
First we compute the precision matrix $M_\vy := \Sigma_\vy^{-1}$:
\begin{equation}
    M_\vy = \frac{1}{1-\gamma^2}      
    \left[
     \begin{array}{cc}
        1 & \gamma \\
        \gamma & 1 
     \end{array} 
    \right]
\end{equation}
and we can apply the lower-bound of the Mills ratio (The equation below Eqn.~5 in~\cite{steck1979lower}) for standarized 2D Gaussian distribution (note that ``$\rho$'' in~\cite{steck1979lower} is $-\gamma$ here): 
\begin{eqnarray}
    R(\vy_0, M_\vy) &:=& 2\pi |\Sigma_\vy|^{1/2}\pr\left(\vy \ge \vy_0\right)\exp\left(\frac{\vy_0^\t M_\vy \vy_0}{2}\right) \\
    &\ge& (1-\gamma^2) \max\left[ r(\kappa_1)r(\kappa_2 + \gamma e(\kappa_1)), r(\kappa_2)r(\kappa_1 + \gamma e(\kappa_2) )\right]  
\end{eqnarray}
where $r(y) > 0$ is the 1D Mills ratio and other terms defined as follows:
\begin{eqnarray}
    r(y) &:=& \int_0^{+\infty} e^{-yt - \frac{t^2}{2}} \dd t \\
    e(y) &:=& \frac{1}{r(y)} - y \\
    \kappa_1 &:=& (1 - \gamma^2)^{-1/2} (\vy_0[1] + \gamma \vy_0[2]) = (1 - \gamma^2)^{-1/2} \sqrt{c}  \\
    \kappa_2 &:=& (1 - \gamma^2)^{-1/2} (\vy_0[2] + \gamma \vy_0[1]) = (1 - \gamma^2)^{-1/2} \gamma \sqrt{c} = \gamma \kappa_1
\end{eqnarray}
Note that for $r(y)$, we have:
\begin{equation}
    r(y) = \int_0^{+\infty} e^{-yt - \frac{t^2}{2}} \dd t = e^{\frac{y^2}{2}} \int_y^{+\infty} e^{-\frac{t^2}{2}} \dd t = \sqrt{2\pi} e^{\frac{y^2}{2}}(1 - r(y)) \label{eq:relation-to-gaussian-cdf}
\end{equation}
where $r(y)$ is the Cumulative Probability Distribution (CDF) for standard Gaussian distribution.  

For simplicity, we can define:
\begin{equation}
    \hat R(c, \gamma) := \max\left[ r(\kappa_1)r(\kappa_2 + \gamma e(\kappa_1)), r(\kappa_2)r(\kappa_1 + \gamma e(\kappa_2) )\right] \label{eq:residue}
\end{equation}

Plug them in, we get:
\begin{equation}
    \pr\left(\vy \ge \vy_0 := \left[
    \begin{array}{c}
        \sqrt{c}\\
        0
    \end{array}
    \right]\right) \ge \frac{1}{2\pi} \mathrm{det}^{1/2}(M_{\vy}) \exp\left(-\frac{1}{2}\vy_0^\t M_\vy \vy_0\right) (1-\gamma^2) \hat R(c, \gamma)
\end{equation}
Note that $\vy_0^\t M_\vy \vy_0 = c$ and $\mathrm{det}^{1/2}(M_{\vy}) = (1 - \gamma^2)^{-1/2}$. Therefore, we have:
\begin{equation}
    \pr\left(\vy \ge \vy_0 := \left[
    \begin{array}{c}
        \sqrt{c}\\
        0
    \end{array}
    \right]\right) \ge \frac{1}{2\pi} \exp\left(-\frac{c}{2}\right) (1-\gamma^2)^{1/2} \hat R(c, \gamma)
\end{equation}
and the conclusion follows.
\end{proof}

\textbf{Remarks.} Obviously, $r(y)$ is a decreasing function. Following Theorem 2~\cite{soranzo2014very}, we have for $y > 0$:
\begin{equation}
    \frac{2}{\sqrt{y^2+4}+y} < r(y) < \frac{4}{\sqrt{y^2+8}+3y} \label{eq:estimation-m}
\end{equation}
and we could estimate the bound of $\hat R(c, \gamma)$ accordingly. Using Eqn.~\ref{eq:relation-to-gaussian-cdf} to estimate the bound would be very difficult numerically since $1 - r(y)$ becomes very close to zero once $y$ becomes slightly large (e.g., when $y \ge 10$).

Note that from Theorem 1 in~\cite{soranzo2014very}, we have a coarser bound for $r(y)$ and $e(y)$ (the bound will blow-up when $y\rightarrow 0$):
\begin{equation}
    \frac{y}{1+y^2} < r(y) < \frac{1}{y}, \quad 0 < e(y) < \frac{1}{y}
\end{equation}
\begin{figure}
    \centering
    \includegraphics[width=0.5\textwidth]{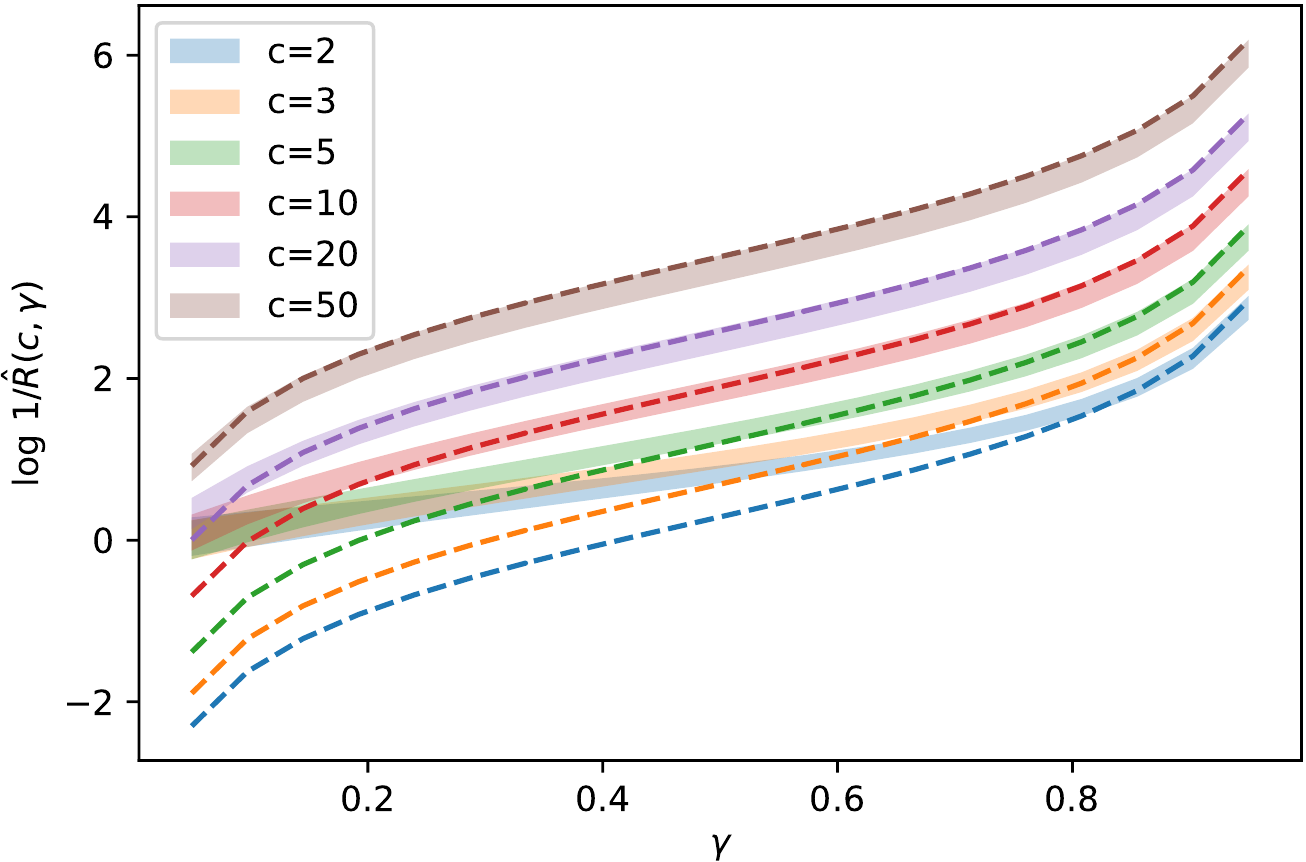}
    \caption{Visualization of $\log 1 / \hat R(c, \gamma)$ with respect to $c$ and $\gamma$. For each color, the shaded area shows the upper and lower limits of $\log 1 / \hat R(c, \gamma)$ computed following Eqn.~\ref{eq:estimation-m}, and the solid line is the asymptotic curve $\hat R(c, \gamma) \sim \frac{1 - \gamma^2}{\gamma c}$ when $c\rightarrow +\infty$ and $\gamma \rightarrow 1$.}
    \label{fig:rcgamma}
\end{figure}

So when $y$ is large, we have $r(y) \sim \frac{1}{y}$. Therefore, when $\gamma$ is close to $1$ and $c$ is large (Fig.~\ref{fig:rcgamma}):
\begin{equation}
\hat R(c, \gamma) \sim \kappa^{-1}_1\kappa^{-1}_2 \sim \frac{1 - \gamma^2}{\gamma c} \label{eq:r-approx}
\end{equation}
Therefore, $\hat R(c, \gamma)$ has the following properties:
\begin{equation}
    \lim_{\gamma\rightarrow 1} \hat R(c, \gamma) = 0, \quad\quad \lim_{c\rightarrow +\infty} \hat R(c, \gamma) = 0 
\end{equation}
Intuitively, the first one is due to the fact that when $\gamma = 1$, $\tilde v^+_j$ and $\tilde v^-_j$ becomes perfectly correlated (i.e., $y_+$ and $y_-$ are perfectly negatively correlated) with each other. In that case, $\pr(\vy\ge \vy_0)$ becomes zero. The second one is due to the fact that $\vy$ as a random variable cannot deviate too far away from its mean value.  

Note that when $\gamma \rightarrow 0$, Eqn.~\ref{eq:r-approx} is not accurate and there is no singularity in $\hat R(c, \gamma)$ near $\gamma=0$, as shown in Fig.~\ref{fig:rcgamma}.  

\subsection{Lucky node at initialization (Sec.~\ref{sec:lucky-node-main-text})}
\label{sec:lucky-node}
\begin{theorem}[Lucky nodes at initialization in deep ReLU networks regarding to SB-\hltm{}]
\label{thm:lottery-ticket-gen-model}
Suppose each element of the weights $W_l$ between layer $l+1$ and $l$ are initialized with $\mathrm{Uniform}\left[-\sigma_w\sqrt{\frac{3}{|\ch{\mu}|}}, \sigma_w\sqrt{\frac{3}{|\ch{\mu}|}}\right]$. Then for any $\mu \in \latents{l+1}$, if 
\begin{equation}
|\nodes{\mu}| \ge \frac{2\pi e^{c/2}}{\sqrt{1-\gamma^2}\hat R(c,\gamma)} \ln 1/\eta
\end{equation}
where 
\begin{equation}
    \gamma := \frac{ \|\bar\vv_{\ch{\mu}}\|^2_2 - \|\vamu\|^2_2 }{\|\bar\vv_{\ch{\mu}} - \vamu\|_2 \|\bar\vv_{\ch{\mu}} + \vamu\|_2 }
\end{equation}
then with probability at least $1 - \eta$, there exists at least one node $j \in\nodes{\mu}$ so that their \emph{pre-activation gap} $\tilde v_j(1) - \tilde v_j(0) = 2\vw_j^\t\vamu > 0$ and the activations satisfy (here $\bar v_j := \frac{1}{2}(v_j(0) + v_j(1))$):
\begin{equation}
    |s_j| \ge 
    \frac{3\sigma_w^2}{4\bar v_j}\left[\frac{1}{|\ch{\mu}|}\sum_{k\in\ch{\mu}} s_k^2\left(\frac{c+6}{6}\rho_{\mu\nu}^2 - 1\right) - \sigma_l^2\right]. 
\end{equation}
Here $\sigma_l^2 := \max_{k\in\ch{\mu}} \var[f_k|z_\nu]$ and $\vamu := [\rho_{\mu\nu(k)}s_k]_{k\in \ch{\mu}}$ is a vector which contains selectivity of all child nodes of $z_\mu$, multiplied by the corresponding polarity $\rho_{\mu\nu(k)}$ for each child $k$ whose latent variable is $\nu=\nu(k)$ (See Fig.~\ref{fig:hltm-appendix}).
\end{theorem}
\begin{proof}
According to our setting, for each node $k\in \nodes{\mu}$, there exists a unique latent variable $z_\nu$ with $\nu = \nu(k)$ that corresponds to it. In the following we omit its dependency on $k$ for brevity. 

We first consider pre-activation $\tilde f_j := \sum_k w_{jk}f_k$ and its expectation with respect to latent variable $z$: 
\begin{equation}
    \tilde v^+_j := \eee{z}{\tilde f_j \Big|z_\mu = 1},\quad  
    \tilde v^-_j := \eee{z}{\tilde f_j \Big|z_\mu = 0}
\end{equation}
Note that for each node $k\in\ch{\mu}$ we compute the conditional expectation of the child node $k$, conditioned on the parent level latent variable $z_\mu$: 
\begin{eqnarray}
  v^+_{\mu k} &:=& \eee{z}{f_k|z_\mu=1} \\
  &=& \sum_{z_\nu} \eee{z}{f_k|z_\nu, z_\mu=1}\pr(z_\nu|z_\mu=1) \\
  &\stackrel{\circle{1}}{=}& \sum_{z_\nu} \eee{z}{f_k|z_\nu}\pr(z_\nu|z_\mu=1) \\
  &=& \frac12(1 + \rho_{\mu\nu})v_k(1) + \frac12 (1 - \rho_{\mu\nu})v_k(0) \\
  &=& \bar v_k + \rho_{\mu\nu}s_k 
\end{eqnarray}
Note that \circle{1} is due to the fact that the activation $f_k$ only depends on $z_\nu$ (and its descendants), and is independent of $z_\mu$ as long as $z_\nu$ is known. Similarly we have $v^-_{\mu k} = \eee{z}{f_k|z_\mu=0} = \bar v_k - \rho_{\mu\nu}s_k$. For convenience, we vectorize it. Let $\vamu := [a_k]_{k\in\ch{\mu}} := [\rho_{\mu\nu}s_k]_{k\in\ch{\mu}}$ and 
\begin{eqnarray}
    \vu_{\ch{\mu}}^+ &:=& V^+_{\mu,\ch{\mu}} := \Big[ \ee{f_k|z_\mu=1} \Big] = \bar\vv_{\ch{\mu}} + \vamu \\ 
    \vu_{\ch{\mu}}^- &:=& V^-_{\mu,\ch{\mu}} := \Big[ \ee{f_k|z_\mu=0} \Big] = \bar\vv_{\ch{\mu}} - \vamu
\end{eqnarray}
Then we have $\tilde v_j^+ = \vw_j^\t \vu_{\ch{\mu}}^+$ and $\tilde v_j^- = \vw_j^\t \vu_{\ch{\mu}}^-$ and the pre-activation gap $\tilde v_j^+ - \tilde v_j^- = 2\vw_j^\t\vamu$.

Note that $\vw_j$ is a random variable with each entry $w_{jk} \sim \mathrm{Uniform}\left[-\sigma_w\sqrt{\frac{3}{|\ch{\mu}|}}, \sigma_w\sqrt{\frac{3}{|\ch{\mu}|}}\right]$. It is easy to verify that $\ee{w_{jk}} = 0$ and $\var[w_{jk}] = \sigma_w^2 / |\ch{\mu}|$. Therefore, for two dimensional vector 
\begin{equation}
   \tilde\vv_j = \left[
    \begin{array}{cc}
        \tilde v_j^+\\
        \tilde v_j^-
    \end{array} 
    \right]
\end{equation}
we can compute its first and second order moments: $\eee{w}{\tilde\vv_j} = \vzero$ and its variance (here we set $\vh := \bar\vv_{\ch{\mu}}$ and $\va := \va_\mu$ for notation brevity, in this case we have $\vu_{\ch{\mu}}^+ = \vh + \va$ and $\vu_{\ch{\mu}}^- = \vh - \va$):
\begin{equation}
    \var_{w}[\tilde\vv_j] = \frac{\sigma_w^2}{|\ch{\mu}|}V_{\mu,\ch{\mu}}V^\t_{\mu,\ch{\mu}} = \frac{\sigma_w^2}{|\ch{\mu}|} \left[
    \begin{array}{c}
         \vu_{\ch{\mu}}^+  \\
         \vu_{\ch{\mu}}^-
    \end{array}\right][\vu_{\ch{\mu}}^+,\vu_{\ch{\mu}}^-] = \frac{\sigma_w^2}{|\ch{\mu}|} \left[
    \begin{array}{c}
        \vh + \va \\
        \vh - \va 
    \end{array}\right][\vh + \va,\vh - \va] 
\end{equation}

Define the positive and negative set (note that $a_k := \rho_{\mu\nu}s_k$):
\begin{equation}
    A_+ = \{k: a_k \ge 0 \}, \quad  A_- = \{k: a_k < 0 \} 
\end{equation}

Without loss of generality, assume that $\sum_{k\in A_+} a^2_k \ge \sum_{k\in A_-} a^2_k$. In the following, we show there exists $j$ with $(v_j^+)^2 - (v_j^-)^2$ is greater than some positive threshold. Otherwise the proof is symmetric and we can show $(v_j^+)^2 - (v_j^-)^2$ is lower than some negative threshold. 

When $|\ch{\mu}|$ is large, by Central Limit Theorem, $\tilde\vv_j$ can be regarded as zero-mean 2D Gaussian distribution with the following covariance matrix:
\begin{equation}
    \tilde\vv_j \sim \mathcal{N}\left(\vzero, 
    \frac{\sigma_w^2}{|\ch{\mu}|}\left[
        \begin{array}{cc}
            \|\vh + \va\|^2_2 & \|\vh\|^2_2 - \|\va\|^2_2 \\
            \|\vh\|^2_2 - \|\va\|^2_2 & \|\vh - \va\|^2_2 
        \end{array}
    \right]
    \right)
\end{equation}

Now we want to lower-bound the following probability. For some $c > 0$:
\begin{equation}
    \pr\left(\tilde v^+_j \ge \frac{\sqrt{c}\sigma_w}{\sqrt{|\ch{\mu}|}}\|\vu^+_{\ch{\mu}}\|\quad\mathrm{and}\quad\tilde v^-_j < 0\right) \label{eq:condition}
\end{equation}

For this, we consider the following transformation to standardize each component of $\tilde \vv_j$: 
\begin{equation}
    y_+ = \tilde v_j^+ \frac{\sqrt{|\ch{\mu}|}}{\sqrt{c}\sigma_w}
    \frac{1}{\|\vh + \va\|_2}, \quad\quad
    y_- = -\tilde v_j^- \frac{\sqrt{|\ch{\mu}|}}{\sqrt{c}\sigma_w}
    \frac{1}{\|\vh - \va\|_2} 
\end{equation}
And the probability we want to lower-bound becomes:
\begin{equation}
    \pr\left(\vy \ge \vy_0 := \left[
    \begin{array}{c}
        \sqrt{c}\\
        0
    \end{array}
    \right]\right)    
\end{equation}

Then $\vy = [y_+, y_-]^\t$ satisfies the condition of Lemma~\ref{lemma:2d-gaussian-lower-bound} with $\gamma$:
\begin{equation}
    \gamma := \frac{\|\vh\|^2_2 - \|\va\|^2_2}{\|\vh+\va\|_2\|\vh-\va\|_2} 
\end{equation}
Note that $\gamma \ge 0$. To see this, we know $\vu_{\ch{\mu}}^+ = \vh + \va \ge \vzero$ and $\vu_{\ch{\mu}}^- = \vh - \va \ge \vzero$ by the definition of $\vu_{\ch{\mu}}^+$ and $\vu_{\ch{\mu}}^-$ since they are both expectation of ReLU outputs (here $\ge$ is in the sense of element-wise). Therefore, we have:
\begin{equation}
    \|\vh\|^2_2 - \|\va\|^2_2 = (\vh + \va)^\t (\vh - \va) \ge 0
\end{equation}
and thus $\gamma \ge 0$.

Therefore, with (note that the residue term $\hat R(c, \gamma)$ is defined in Lemma~\ref{lemma:2d-gaussian-lower-bound}):
\begin{equation}
|\nodes{\mu}| \ge \frac{2\pi e^{c/2}}{\sqrt{1-\gamma^2}\hat R(c,\gamma)} \ln 1/\eta
\end{equation}
we have:
\begin{equation}
    |\nodes{\mu}| \ge \frac{\ln 1/\eta}{\pr\left(\vy \ge \vy_0\right)} \ge \frac{\ln 1/\eta}{-\ln\left[1 - \pr\left(\vy \ge \vy_0\right)\right]}
\end{equation}
Or 
\begin{equation}
    1 - (1 - \pr\left(\vy \ge \vy_0\right))^{|\nodes{\mu}|} \ge 1 - \eta
\end{equation}
which means that with at least $1 - \eta$ probability, there exists at least one node $j\in \nodes{\mu}$ and $\vw_j$, so that Eqn.~\ref{eq:condition} holds. 

In this case, we have:
\begin{equation}
    \tilde v^+_j = \vw_j^\t\vu^+_{\ch{\mu}} \ge \frac{\sqrt{c}\sigma_w}{\sqrt{|\ch{\mu}|}}\|\vu^+_{\ch{\mu}}\|, \quad\quad\tilde v^-_j = \vw_j^\t\vu^-_{\ch{\mu}} < 0
\end{equation}
Since $\bar \vv_{\ch{\mu}} \ge 0$ (all $f_k$ are after ReLU and non-negative), this leads to:
\begin{equation}
    \tilde v^+_j \ge \frac{\sqrt{c}\sigma_w}{\sqrt{|\ch{\mu}|}}\|\vu^+_{\ch{\mu}}\|\ge \frac{\sqrt{c}\sigma_w}{\sqrt{|\ch{\mu}|}}\sqrt{\sum_{k\in A_+} a_k^2}\ge \sigma_w\sqrt{\frac{c}{2|\ch{\mu}|}\sum_{k\in\ch{\mu}} \rho^2_{\mu\nu}s_k^2} 
\end{equation}
By Jensen's inequality, we have (note that $\psi(x) := \max(x,0)$ is the ReLU activation): 
\begin{eqnarray}
v^+_j &=& \eee{z}{f_j|z_\mu=1} = \eee{z}{\psi(\tilde f_j)|z_\mu=1} \\
&\ge& \psi\left(\eee{z}{\tilde f_j\Big|z_\mu=1}\right) = \psi(\tilde v^+_j) \ge \sigma_w\sqrt{\frac{c}{2|\ch{\mu}|}\sum_{k\in\ch{\mu}} \rho^2_{\mu\nu}s_k^2} \label{eq:pos-bound}
\end{eqnarray} 

On the other hand, we also want to compute $v_j^- := \eee{z}{f_j|z_\mu=0}$ using sharpened Jensen's inequality (Lemma~\ref{lemma:sharpened-jensen}). For this we need to compute the conditional covariance $\var_z[\tilde f_j|z_\mu]$:
\begin{eqnarray}
    \var_z[\tilde f_j|z_\mu] &\stackrel{\circle{2}}{=}& \sum_k w^2_{jk} \var_z[f_k|z_\mu] \stackrel{\circle{3}}{\le} \frac{3\sigma_w^2}{|\ch{\mu}|}\sum_k  \var_z[f_k|z_\mu] \\
   &=& \frac{3\sigma_w^2}{|\ch{\mu}|}\sum_k \left(\eee{z_\nu|z_\mu}{\var[f_k|z_\nu]} +  \var_{z_\nu|z_\mu}[\eee{z}{f_k|z_\nu}]\right) \\
   &\le& 3\sigma_w^2 \left(\sigma^2_l + \frac{1}{|\ch{\mu}|}\sum_k  \var_{z_\nu|z_\mu}[\eee{z}{f_k|z_\nu}]\right)
\end{eqnarray}
Note that \circle{2} is due to conditional independence: $f_k$ as the computed activation, only depends on latent variable $z_\nu$ and its descendants. Given $z_\mu$, all $z_\nu$ and their respective descendants are independent of each other and so does $f_k$. \circle{3} is due to the fact that each $w_{jk}$ are sampled from uniform distribution and $|w_{jk}|\le \sigma_w\sqrt{\frac{3}{|\ch{\mu}|}}$. 

Here $\var_{z_\nu|z_\mu}[\eee{z}{f_k|z_\nu}] = s_k^2(1 - \rho^2_{\mu\nu})$ can be computed analytically. It is the variance of a binomial distribution: with probability $\frac{1}{2}(1+\rho_{\mu\nu})$ we get $v^+_k$ otherwise get $v^-_k$. Therefore, we finally have: 
\begin{equation}
    \var_z[\tilde f_j|z_\mu] \le 3\sigma_w^2\left(\sigma_l^2 + \frac{1}{|\ch{\mu}|}\sum_k s_k^2(1-\rho^2_{\mu\nu})\right)
\end{equation}

As a side note, using Lemma~\ref{lemma:var-squashing}, since ReLU function $\psi$ has Lipschitz constant $\le 1$ (empirically it is smaller), we know that:
\begin{equation}
    \var_z[f_j|z_\mu] \le 3\sigma_w^2\left(\sigma_l^2 + \frac{1}{|\ch{\mu}|}\sum_k s_k^2(1-\rho^2_{\mu\nu})\right)
\end{equation}
Finally using Lemma~\ref{lemma:sharpened-jensen} and $\tilde v^-_j := \eee{z}{\tilde f_j|z_\mu=0} < 0$, we have:
\begin{eqnarray}
v^-_j &=& \eee{z}{f_j|z_\mu=0} = \eee{z}{\psi(\tilde f_j)|z_\mu=0} \\
&\le& \psi\left(\eee{z}{\tilde f_j\Big|z_\mu=0}\right) + \sqrt{\var_z[\tilde f_j|z_\mu=0]} \\
&=&\sqrt{\var_z[\tilde f_j|z_\mu=0]} \\
&\le& \sigma_w \sqrt{3\sigma_l^2 + \frac{3}{|\ch{\mu}|}\sum_k s_k^2(1-\rho_{\mu\nu}^2)} \label{eq:neg-bound} 
\end{eqnarray} 

Combining Eqn.~\ref{eq:pos-bound} and Eqn.~\ref{eq:neg-bound}, we have a bound for $\lambda_j$:
\begin{equation}
    (v^+_j)^2 - (v^-_j)^2 \ge 3\sigma_w^2\left[\frac{1}{|\ch{\mu}|}\sum_k s_k^2 \left(\frac{c+6}{6}\rho_{\mu\nu}^2 - 1\right) - \sigma_l^2\right] 
\end{equation}
Note that $|(v^+_j)^2 - (v^-_j)^2| = |(v^+_j + v^-_j)(v^+_j - v^-_j)| = 4|\bar v_j s_j|$ and the conclusion follows.
\end{proof}
\textbf{Remark}. Intuitively, this means that with large polarity $\rho_{\mu\nu}$ (strong top-down signals), randomly initialized over-parameterized ReLU networks yield selective neurons, if the lower layer also contains selective ones. 
For example, when $\sigma_l = 0$, $c = 9$, if $\rho_{\mu\nu} \ge 63.3\%$ then there is a gap between expected activations $v_j(1)$ and $v_j(0)$, and the gap is larger when the selectivity in the lower layer $l$ is higher. With the same setting, if $\gamma = 0.2$, we could actually compute the range of $|\nodes{\mu}|$ specified by Theorem~\ref{thm:lucky-node}, using Eqn.~\ref{eq:estimation-m}:
\begin{equation}
    1099 \le |\nodes{\mu}| \le 1467
\end{equation}
while in practice, we don't need such a strong over-parameterization.

For $l = 0$ (the lowest layer), $\{v_k\}$ are themselves observable leaf latent variables and are selective by definition ($|s_k| \equiv 1$) and the conditional variance $\sigma^2_l \equiv 0$. In the higher level, $|s_j|$ often goes down, since only a few child nodes will be selective (with large $s_k^2$), and their contribution are divided by $|\ch{\mu}|$. These insights are confirmed in Tbl.~\ref{tbl:hltm} and more experiments in Sec.~\ref{exp:hltm-experiment}.

\subsection{Training dynamics (Sec.~\ref{sec:hltm-training})}
\begin{lemma}[Activation covariance in TR-\hltm{} (Def.~\ref{definition:trans-hltm})] 
\label{lemma:act-cov-tr-hltm}
The variance of augmentation-mean activation $\eee{z}{\vf_{\ch{\mu}}|z_0}$ of the children of intermediate nodes $\nodes{\mu}$ is: 
\begin{equation}
    \var_{z_0}[\eee{z}{\vf_{\ch{\mu}}|z_0}] = V_{0,\ch{\mu}}^\t (P_0 - P_0^\t\vone\vone^\t P_0) V_{0,\ch{\mu}}
\end{equation}
where $P_0 := \mathrm{diag}[\pr(z_0)]$, $V_{0,\ch{\mu}} := [\vv_{0k}]_{k\in \ch{\mu}}$ and each of its column $\vv_{0k} :=  [\eee{z}{f_k|z_0}] \in \rr^{m_0}$. Note that each element of $\vv_{0k}$ runs through different instantiation of categorical variable $z_0$.
\end{lemma}
\begin{proof}
For TR-\hltm{} (Def.~\ref{definition:trans-hltm}) which is an extension of SB-\hltm{} for categorical latent variables (Lemma~\ref{lemma:tr-sb-relation}), for each node $k\in \ch{\mu}$ we have: 
\begin{eqnarray}
    \vv_{0k} &:=& [\eee{z}{f_k|z_0}] = \sum_{z_\nu} \eee{z}{f_k|z_{\nu}}\pr(z_\nu|z_0) = P_{0\nu}\vv_k \\
    &=& \left(\frac{1}{\cardinal\nu}\vone_0\vone_\nu^\t + C_{0\nu} \right)\vv_k \\
    &=& \frac{1}{\cardinal\nu}\vone_0\vone_\nu^\t\vv_k + C_{0\nu}\vv_k \label{eq:vok}
\end{eqnarray}
For a categorical distribution \[
(\pr(z_0=0), u(0)),\quad (\pr(z_0=1), u(1)),\quad\ldots, \quad (\pr(z_0=\cardinal0-1), u(\cardinal0-1))
\]
With $P_0 := \mathrm{diag}[\pr(z_0)]$, the mean is $\eee{z_0}{u} = \vone^\t P_0\vu$ and its covariance can be written as (here $\vone = \vone_0$): 
\begin{equation}
    \var_{z_0}[u] = (\vu - \vone\vone^\t P_0\vu)^\t P_0 (\vu - \vone\vone^\t P_0\vu) = \vu^\t (P_0-P_0\vone\vone^\t P_0)\vu 
\end{equation}
Note that each column of $V_{0,\ch{\mu}}$ is $\vv_{0k}$. Setting $\vu = \vv_{0k}$ and we have:
\begin{equation}
    \var_{z_0}[\eee{z}{\vf_{\ch{\mu}}|z_0}] = V_{0,\ch{\mu}}^\t (P_0 - P_0^\t\vone\vone^\t P_0) V_{0,\ch{\mu}}
\end{equation}
\end{proof}

\begin{theorem}[Activation covariance in SB-\hltm]
$\var_{z_0}[\eee{z'}{\vf_{\ch{\mu}}|z_0}] = o_\mu\vamu\vamu^\t$. Here $o_\mu := \rho^2_{0\mu} (1-\rho_0^2)$. Also $\lim_{L\rightarrow +\infty} \rho_{0\mu} = 0$. Therefore, when $L\rightarrow +\infty$, the covariance $\var_{z_0}[\eee{z'}{\vf_{\ch{\mu}}|z_0}]$ becomes zero as well.
\end{theorem}
\begin{proof}
We simply apply Lemma~\ref{lemma:act-cov-tr-hltm} to SB-\hltm{}. Note here we define $\rho_0 := \pr(z_0=1)-\pr(z_0=0)$, and $\vq := [-1,1]^\t$ so we have:
\begin{equation}
    P_0 - P_0 \vone\vone^\t P_0 = \frac{1}{4}(1 - \rho^2_0) \vq\vq^\t
\end{equation}

Therefore, we have:
\begin{equation}
    \var_{z_0}[\eee{z}{\vf_{\ch{\mu}}|z_0}] = \frac{1}{4}(1 - \rho^2_0) V_{0,\ch{\mu}}^\t \vq\vq^\t V_{0,\ch{\mu}} \quad\in \rr^{|\ch{\mu}| \times |\ch{\mu}|}   
\end{equation}
This is a square matrix of size $|\ch{\mu}|$-by-$|\ch{\mu}|$. Now let's check each entry. For entry $(k, k')$ (where $k,k' \in \ch{\mu}$), it is
\begin{equation}
    \frac{1}{4}(1 - \rho^2_0)(\vv_{0k}^\t \vq) (\vv^\t_{0k'}\vq)^\t
\end{equation}
So it reduces to computing $\vv_{0k}^\t \vq$. Note that following Eqn.~\ref{eq:vok}:
\begin{equation}
    \vq^\t\vv_{0k} = \frac{1}{\cardinal\nu}\vq^\t\vone_0\vone_\nu^\t\vv_k + \vq^\t C_{0\nu}\vv_k = \vq^\t C_{0\nu}\vv_k
\end{equation}
since $\vq^\t\vone_0 = 0$. In SB-\hltm, according to Lemma~\ref{lemma:tr-sb-relation}, all $C_{\mu\nu} = \frac{1}{2}\rho_{\mu\nu} \vq\vq^\t$ and we could compute $C_{0\nu}\vv_k$:
\begin{equation}
    C_{0\nu}\vv_k = \frac{1}{2}\rho_{0\nu}\vq\vq^\t\vv_k = \rho_{0\nu}\frac{1}{2}(v_k(1) -v_k(0))\vq =  \rho_{0\nu}s_k
    \vq
\end{equation}
Therefore, we have $\vq^\t\vv_{0k} = \vq^\t C_{0\nu}\vv_k = 2\rho_{0\nu}s_k$ since $\vq^\t\vq = 2$.

On the other hand, according to Eqn.~\ref{eq:close-multiplication}, we have: 
\begin{equation}
    \rho_{0\nu} := 
\prod_{0,\ldots,\alpha,\beta,\ldots,\nu}\rho_{\alpha\beta} = \rho_{0\mu}\rho_{\mu\nu} 
\end{equation}
This is due to the fact that due to tree structure, the path from $z_0$ to all child nodes in $\ch{\mu}$ must pass $z_\mu$. Therefore, we have $\vq^\t\vv_{0k} = 2\rho_{0\mu}\rho_{\mu\nu}s_k = 2\rho_{0\mu} \vamu[k]$, which gives for entry $(k,k')$:
\begin{equation}
    \frac{1}{4}(1 - \rho^2_0)(\vv_{0k}^\t \vq) (\vv^\t_{0k'}\vq)^\t = \rho^2_{0\mu}(1 - \rho_0^2) \vamu[k] \vamu[k'] 
\end{equation}

Put them together, the covariance operator is:
\begin{equation}
 \var_{z_0}[\eee{z}{\vf_{\ch{\mu}}|z_0}] = \rho^2_{0\mu}(1-\rho_0^2)\vamu\vamu^\t 
\end{equation}

When $L \rightarrow +\infty$, from Lemma~\ref{lemma:probability-transition} and its remark, we have:
\begin{equation}
    \rho_{0\mu} := 
\prod_{0,\ldots,\alpha,\beta,\ldots,\mu}\rho_{\alpha\beta} \rightarrow 0 
\end{equation}
and thus the covariance becomes zero as well.
\end{proof}

\clearpage

\begin{figure}
    \centering
    \includegraphics[width=\textwidth]{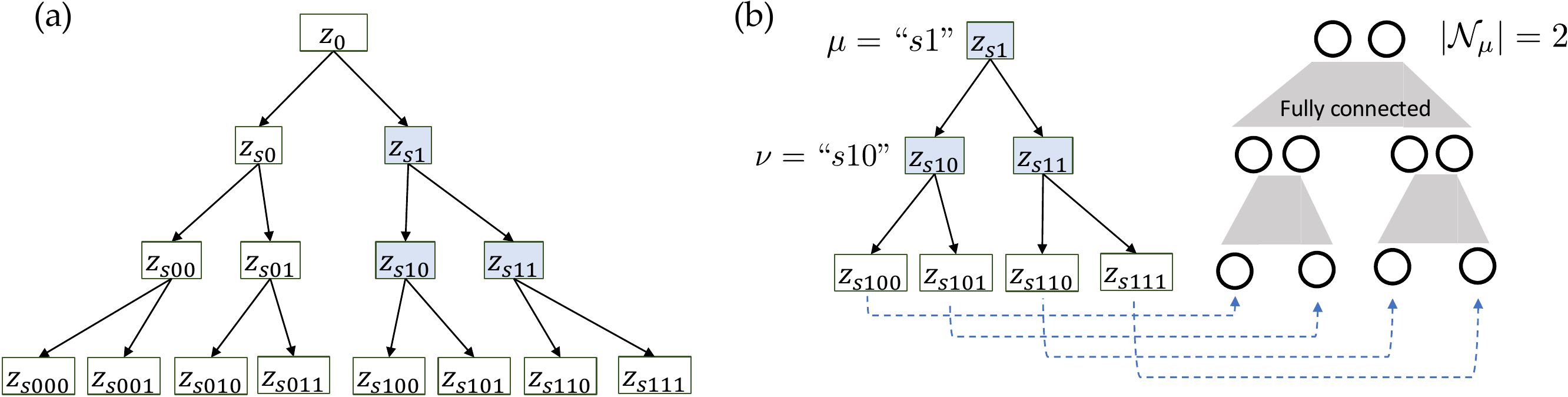}
    \caption{The Hierarchical Latent Tree Model (\hltm{}) used in our experiments (Sec.~\ref{exp:hltm-experiment} and Sec.~\ref{sec:experiment}).} 
    \label{fig:hltm-exp-setup}
\end{figure}

\begin{table}[h!]
\centering
\small
\label{tbl:hltm}
\caption{\small Normalized Correlation between the intermediate layer latent variable in SB-\hltm{} and corresponding group of nodes in deep ReLU networks ($L = 5$) trained with SimCLR with NCE loss. With more over-parameterization, correlations are higher with lower std on 10 trials at both init and end of training, }
\label{tbl:intermediate-layer}
\setlength{\tabcolsep}{1pt}

\begin{tabular}{c||c|c|c||c|c|c}
    \textbf{Latent node} $z_{0}$ (Root) & \multicolumn{3}{c||}{$|\nodes{\mu}|=2$} & \multicolumn{3}{c}{$|\nodes{\mu}|=10$} \\
    $\rho_{\mu\nu}$  & Initial & After epoch 1 & Converged & Initial & After epoch 1 & Converged \\
    \hline
$\sim \mathrm{Uniform}[0.7,1]$ &  $0.35\pm 0.20$ &   $0.39\pm 0.20$  &   $0.62\pm 0.29$ &  $0.62\pm 0.10$  &  $0.82\pm 0.06$  &   $0.85\pm 0.05$ \\
$\sim \mathrm{Uniform}[0.8,1]$ &  $0.48\pm 0.23$ &   $0.57\pm 0.30$  &   $0.72\pm 0.31$ &  $0.75\pm 0.08$  &  $0.90\pm 0.03$  &   $0.91\pm 0.03$ \\
$\sim \mathrm{Uniform}[0.9,1]$ &  $0.66\pm 0.28$ &   $0.73\pm 0.30$  &   $0.80\pm 0.29$ &  $0.88\pm 0.05$  &  $0.95\pm 0.02$  &   $0.96\pm 0.01$  
\end{tabular}

\vspace{0.1in}

\begin{tabular}{c||c|c|c||c|c|c}
    \textbf{Latent node} $z_{s0}$ & \multicolumn{3}{c||}{$|\nodes{\mu}|=2$} & \multicolumn{3}{c}{$|\nodes{\mu}|=10$} \\
    $\rho_{\mu\nu}$  & Initial & After epoch 1 & Converged & Initial & After epoch 1 & Converged \\
    \hline
$\sim \mathrm{Uniform}[0.7,1]$ &          $0.52\pm 0.20$  &   $0.53\pm 0.20$  &    $0.62\pm 0.21$ &  $0.68\pm 0.07$  &   $0.83\pm 0.06$  &    $0.86\pm 0.05$ \\
$\sim \mathrm{Uniform}[0.8,1]$ &          $0.65\pm 0.19$  &   $0.65\pm 0.18$  &    $0.70\pm 0.14$ & $0.79\pm 0.05$   &  $0.90\pm 0.03$  &    $0.91\pm 0.03$ \\
$\sim \mathrm{Uniform}[0.9,1]$ &          $0.80\pm 0.13$  &   $0.81\pm 0.12$  &    $0.84\pm 0.09$ &  $0.90\pm 0.02$  &   $0.95\pm 0.02$  &    $0.96\pm 0.01$
\end{tabular}

\vspace{0.1in}

\begin{tabular}{c||c|c|c||c|c|c}
    \textbf{Latent node} $z_{s1}$ & \multicolumn{3}{c||}{$|\nodes{\mu}|=2$} & \multicolumn{3}{c}{$|\nodes{\mu}|=10$} \\
    $\rho_{\mu\nu}$  & Initial & After epoch 1 & Converged & Initial & After epoch 1 & Converged \\
    \hline
$\sim \mathrm{Uniform}[0.7,1]$ &         $0.45\pm 0.13$  &   $0.50\pm 0.20$  &    $0.64\pm 0.22$ & $0.64\pm 0.10$  &   $0.78\pm 0.08$   &   $0.82\pm 0.06$ \\ 
$\sim \mathrm{Uniform}[0.8,1]$ &         $0.60\pm 0.16$  &   $0.62\pm 0.16$  &    $0.79\pm 0.20$ & $0.76\pm 0.07$  &   $0.87\pm 0.04$   &   $0.89\pm 0.04$ \\
$\sim \mathrm{Uniform}[0.9,1]$ &         $0.72\pm 0.08$  &    $0.78\pm 0.12$ &     $0.86\pm 0.09$ & $0.88\pm 0.04$  &   $0.94\pm 0.02$   &   $0.95\pm 0.02$
\end{tabular}

\vspace{0.1in}

\begin{tabular}{c||c|c|c||c|c|c}
    \textbf{Latent node} $z_{s101}$ & \multicolumn{3}{c||}{$|\nodes{\mu}|=2$} & \multicolumn{3}{c}{$|\nodes{\mu}|=10$} \\
    $\rho_{\mu\nu}$  & Initial & After epoch 1 & Converged & Initial & After epoch 1 & Converged \\
    \hline
$\sim \mathrm{Uniform}[0.7,1]$ &         $0.66\pm 0.24$   &    $0.69\pm 0.24$    &    $0.72\pm 0.24$ & $0.87\pm 0.06$   &    $0.90\pm 0.04$    &    $0.92\pm 0.04$ \\
$\sim \mathrm{Uniform}[0.8,1]$ &         $0.74\pm 0.25$   &    $0.75\pm 0.24$    &    $0.77\pm 0.25$ & $0.91\pm 0.04$   &   $0.93\pm 0.03$    &    $0.95\pm 0.02$ \\
$\sim \mathrm{Uniform}[0.9,1]$ &         $0.83\pm 0.22$   &    $0.84\pm 0.22$    &    $0.86\pm 0.22$ &  $0.96\pm 0.02$  &    $0.97\pm 0.01$   &     $0.97\pm 0.01$
\end{tabular}

\vspace{0.1in}

\begin{tabular}{c||c|c|c||c|c|c}
    \textbf{Latent node} $z_{s0111}$ & \multicolumn{3}{c||}{$|\nodes{\mu}|=2$} & \multicolumn{3}{c}{$|\nodes{\mu}|=10$} \\
    $\rho_{\mu\nu}$  & Initial & After epoch 1 & Converged & Initial & After epoch 1 & Converged \\
    \hline
$\sim \mathrm{Uniform}[0.7,1]$ &    $0.83\pm 0.19$   &     $0.83\pm 0.19$   &      $0.84\pm 0.19$  & $0.93\pm 0.03$   &     $0.93\pm 0.03$   &      $0.93\pm 0.03$ \\ 
$\sim \mathrm{Uniform}[0.8,1]$ &    $0.87\pm 0.16$   &     $0.88\pm 0.15$   &      $0.88\pm 0.15$  & $0.95\pm 0.02$   &     $0.95\pm 0.02$   &      $0.96\pm 0.02$ \\
$\sim \mathrm{Uniform}[0.9,1]$ &    $0.93\pm 0.10$   &     $0.93\pm 0.10$   &      $0.93\pm 0.10$  & $0.98\pm 0.01$   &     $0.98\pm 0.01$   &      $0.98\pm 0.01$ 
\end{tabular}
\end{table}

\section{Additional Experiments}
\subsection{HLTM}
\label{exp:hltm-experiment}
\textbf{Experiment Setup}. The construction of our Symmetric Binary Hierarchical Latent Tree Models (SB-\hltm) is shown in Fig.~\ref{fig:hltm-exp-setup}(a). We construct a complete binary tree of depth $L$. The class/sample-specific latent $z_0$ is at the root, while other nuisance latent variables are labeled with a binary encoding (e.g., $\mu = ``\mathrm{s010}"$ for a $z_\mu$ that is the left-right-left child of the root $z_0$). The polarity $\rho_{\mu\nu}$ is random sampled from a uniform distribution: $\rho_{\mu\nu} \sim \mathrm{Uniform}[\mathrm{delta\_lower}, 1]$ at each layer, where $\mathrm{delta\_lower}$ is a hyperparameter.  

On the neural network side, we use deep ReLU networks. We train with NCE loss ($\lnce^\tau$) with $\tau = 0.1$ and $H = 1$. $\ell_2$ normalization is used in the output layer. Simple SGD is used for the training with learning rate $0.1$. We generate a dataset using \hltm{} with $N = 64000$ samples. The training batchsize is 128. The model usually is well converged after 50 epochs.  

\textbf{Normalized correlation between latents in SB-\hltm{} and hidden nodes in neural networks}. We provide additional table (Tbl.~\ref{tbl:intermediate-layer}) that shows that besides the top-most layers, the normalized correlation (NC) between the hidden layer of the deep ReLU network and the intermediate latent variables of the hierarchical tree model is also strong at initialization and will grow over time. In particular, we can clearly see the following several trends:
\begin{itemize}
    \item[(1)] Top-layer latent is less correlated with the top-layer nodes in deep networks. This shows that top-layer latents are harder to learn (and bottom-layer latents are easier to learn since they are closer to the corresponding network layers). 
    \item[(2)] There is already a non-trivial amount of NC at initialization. Furthermore, NC is higher in the bottom-layer, since they are closer to the input. Nodes in the lowest layer (leaves) would have perfect NC (1.0), since they are identical to the leaves of the latent tree models. 
    \item[(3)] Overparameterization ($|\ch{\mu}| > 1$) helps for both the initial NC and NC after during training. Furthermore, over-parameterization seems to also greatly accelerate the improvement of NC. Indeed, NC after 1 epoch grows much faster in $|\ch{\mu}| = 10$ (10x over-parameterization) than in $|\ch{\mu}| = 2$ (2x over-parameterization). 
\end{itemize}
All experiments are repeated 10 times and its mean and standard derivation are reported. 

\textbf{Growth of the norm of the covariance operator}. We also check how the norm of the covariance operator ($\covop{}$) changes during training in different situations. The results are shown in Fig.~\ref{fig:hltm-top-node} and Fig.~\ref{fig:hltm-bottom-node}. We could see that norm of the covariance operator indeed go up, showing that it gets amplified during training (even if the output is $\ell_2$-normalized). For each experiment configuration, we run 30 random seeds and take the \emph{median} of the norm of the covariance operator. Note that due to different initialization of the network, the standard derivation can be fairly large and is not shown in the figure. 

\begin{figure}
    \centering
    \includegraphics[width=\textwidth]{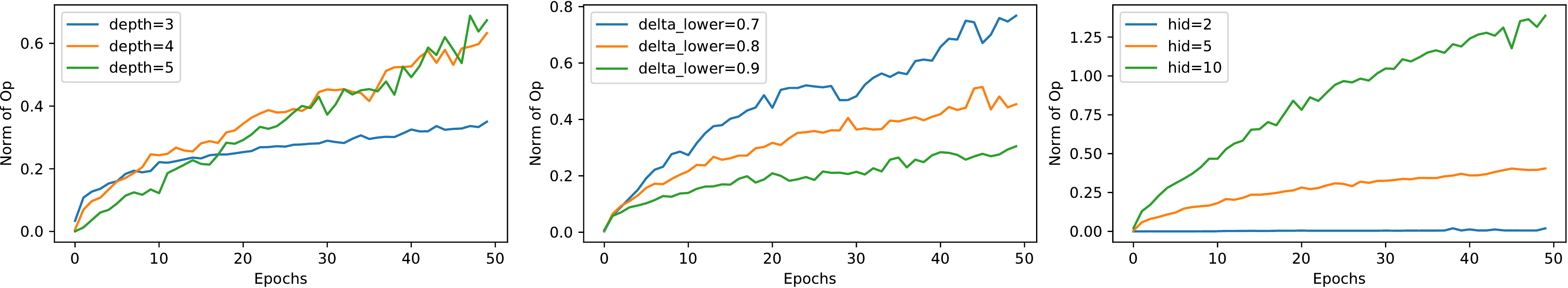}
    \includegraphics[width=\textwidth]{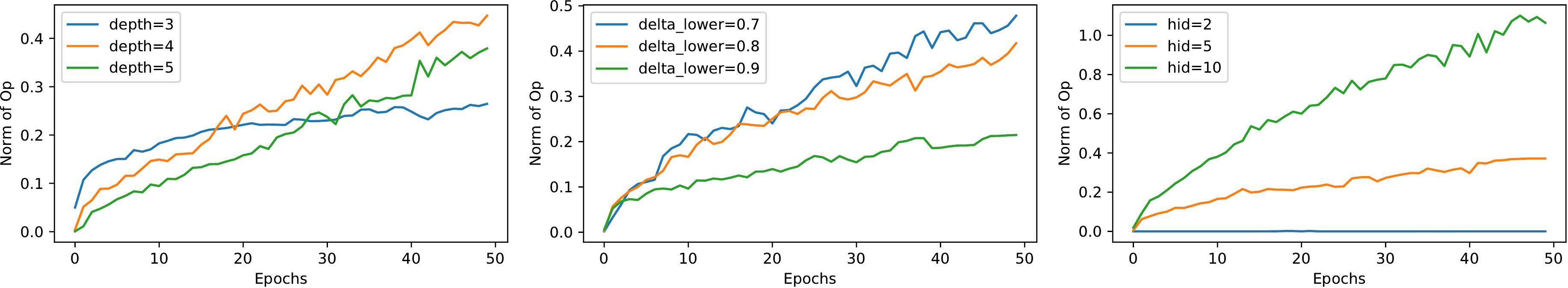}
    \caption{\small The Frobenius norm of the covariance operator $\covop$ changes over training (\emph{median} Frobenius norm over 30 random trials), under different factors: depth $L$, sample range of $\rho_{\mu\nu}$ ($\rho_{\mu\nu} \sim \mathrm{Uniform}[\mathrm{delta\_lower}, 1]$) and over-parameterization $|\nodes{\mu}|$. \textbf{Top row}: covariance operator of immediate left child $z_{s0}$ of the root node $z_0$; \textbf{Bottom row}: covariance operator of immediate right child $z_{s1}$ of the root node $z_0$.}
    \label{fig:hltm-top-node}
\end{figure}

\begin{figure}
    \centering
    \includegraphics[width=\textwidth]{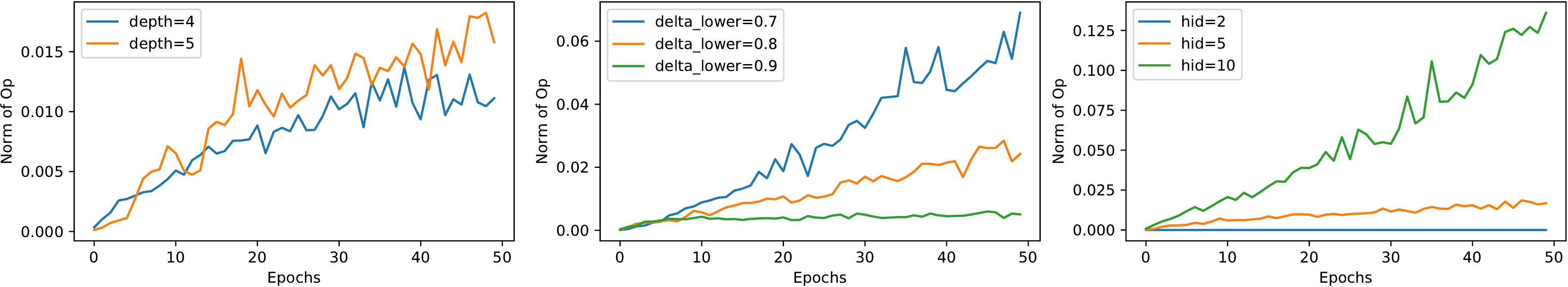}
    \includegraphics[width=\textwidth]{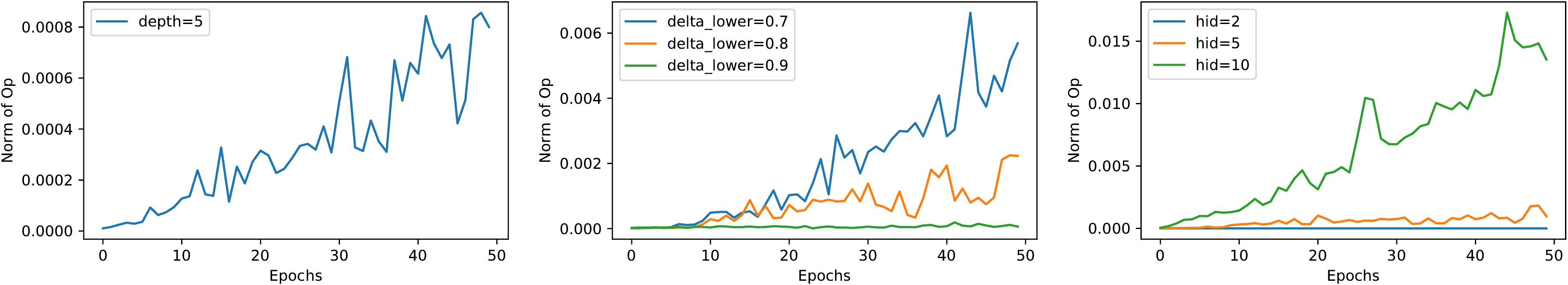}
    \caption{Ablation of how the Frobenius norm of the covariance operator $\covop$ changes over training (\emph{median} Frobenius norm over 30 random trials). Same setting as Fig.~\ref{fig:hltm-top-node} but focus on lower level. Note that since we have used $\ell_2$ normalization at the topmost layer, the growth of the covariance operator is likely not due to the growth of the weight magnitudes, but due to more discriminative representations of the input features $\vf_\mu$ with respect to distinct $z_0$. \textbf{Top row}: covariance operator of left-right-left latent variable ($z_{s010}$) from the root node $z_0$; \textbf{Bottom row}: covariance operator right-left-right-left ($z_{s1100}$) latent variable from the root node $z_0$.}
    \label{fig:hltm-bottom-node}
\end{figure}

\subsection{Experiments Setup}
\label{sec:byol-exp-setup}
For all STL-10~\cite{coates2011analysis} and CIFAR-10~\cite{krizhevsky2009learning} task, we use ResNet18~\cite{resnet} as the backbone with a two-layer MLP projector. We use Adam~\cite{kingma2014adam} optimizer with momentum $0.9$ and no weight decay. For CIFAR-10 we use $0.01$ learning rate and for STL-10 we use $0.001$ learning rate. The training batchsize is $128$. 


\end{document}